\pdfoutput=1
\documentclass[11pt]{article}
\usepackage{graphicx} 
\usepackage{diagbox}
\usepackage{authblk}
\makeatletter
\renewcommand\AB@authnote[1]{\textsuperscript{#1}\hspace{5pt}}

\newcommand{\reg}{\text{Reg}}
\newcommand{\dreg}{\text{DynamicReg}}
\makeatother
\usepackage{hyperref}
\usepackage{algorithm}
\usepackage{algpseudocode}
\usepackage{notation-2}
\usepackage{arxiv-2}
\usepackage{amssymb}
\usepackage{mathrsfs}
\usepackage{float}
\usepackage{setspace}
\usepackage{lmodern}        

\renewcommand{\tilde}[1]{\widetilde{#1}}

\renewcommand{\hat}[1]{\widehat{#1}}

\newcommand{\pll}{\kern 0.3em/\kern -0.9em /\kern 0.3em}

\title{\normalfont Optimizing the Preconditioner: A Black-box Online-to-Nonconvex Conversion with Static Regret Minimization Oracles
}
\begin{centering}
     \author[1]{Haichen Hu\thanks{{Email: \texttt{huhc@mit.edu}}}}
      \author[3]{David Simchi-Levi\thanks{{Email: \texttt{dslevi@mit.edu}}}}
\end{centering}
\affil[1]{Center for Computational Science and Engineering, MIT}
\affil[3]{Institute for Data, Systems, and Society, MIT}
\providecommand{\cB}{\mathcal{B}}\providecommand{\cC}{\mathcal{C}} 
\makeatletter 
\@ifundefined{corollary}{\@ifundefined{c@theorem}{\newtheorem{corollary}{Corollary}[section]}{\newtheorem{corollary}[theorem]{Corollary}}}{}
\makeatother
\begin{document}
\maketitle
\singlespacing
\begin{abstract}
Stochastic nonconvex optimization is central to training deep networks and LLMs
in modern machine learning. We give a black-box reduction from stochastic
nonconvex optimization to ordinary static regret minimization in online convex
optimization (OCO), addressing the open problem posed by \citet{chen2024open}.
Our reduction maintains a predictable gradient tracker, while a black-box online
learner $\cA$ selects a preconditioner that transforms this tracker into the update
direction. Given a \(\beta\)-smooth function with initial gap at most \(M\) and an
unbiased gradient oracle with variance bounded by \(\sigma^2\), we bound the
expected average squared gradient norm by
\(O(\sigma\sqrt{M\beta/T}+\sqrt{M\beta}\mathrm{Reg}_T(\mathcal A)/T+\frac{M\beta}{T})\), where $\mathrm{Reg}_T(\mathcal A)$ is the static regret of $\cA$ against the identity preconditioner.
Thus, any OCO oracle with \(O(\sqrt T)\) regret recovers the classical
\(O(T^{-1/2})\) convergence rate. We also prove that the additive variance term cannot be removed: every conversion making a constant number of stochastic-oracle queries per round, whose guarantee is stated through static regret against a comparator that the online learner can play, must contain an additive term of order \(\sigma\sqrt{M\beta/T}\) for sufficiently large \(T\). Moreover, for every online learner, our conversion bound is tight up to universal constants. We also extend the framework to nonsmooth
nonconvex optimization, still relying only on ordinary static regret,
and attain the optimal convergence rate for Goldstein-type stationarity. Finally, we conduct numerical experiments on nonconvex objectives to illustrate how the reduction exploits online-selected preconditioners while using the same stochastic-oracle budget as stochastic
gradient descent.
\end{abstract}

{\sloppy
\par}

\section{Introduction}
Stochastic nonconvex optimization has emerged as a fundamental area of modern optimization, motivated by both its theoretical challenges and its broad practical relevance \citep{carmon2018accelerated,burer2012non}. Problems of this form arise in numerous domains, including signal processing \citep{saab2008stable}, logistics \citep{velasquez2019large}, network science \citep{di2016next}, and operations management \citep{azizan2020optimal}. More recently, stochastic nonconvex optimization has become particularly important in modern machine learning, where it underlies the training of deep neural networks and large language models (LLMs) \citep{cui2020multicomposite,liang2024memory,jordan2024muon,zhao2024galore}. Many optimizers have been developed and proposed for solving stochastic nonconvex optimization problems, such as Adam \citep{kingma2014adam}, Muon \citep{liu2025muon}, Shampoo \citep{gupta2018shampoo}, Lion \citep{chen2023symbolic}, Pion \citep{shi2026pion}, and so on.

To solve such stochastic optimization problems, a broad range of first-order methods has been developed to transform noisy gradient observations into effective update directions, including stochastic gradient descent \citep{ghadimi2013stochastic}, momentum-based methods, variance-reduced algorithms, and adaptive gradient methods \citep{defossezsimple}. A particularly influential recent approach to stochastic nonconvex optimization is the online-to-nonconvex conversion. Broadly speaking, an online-to-nonconvex conversion takes an online learning algorithm with a regret guarantee and transforms it into an algorithm for an offline nonconvex optimization problem. It typically constructs a sequence of convex, often linear, online losses from stochastic gradient information and then translates the regret of the online learner into a stationarity guarantee for the original nonconvex objective.

Related episodic convex-to-nonconvex reductions can be traced back to \citet{paquette2018catalyst,wang2019stochastic,agarwal2019efficient}, whose reductions implement each iteration of nonconvex optimization through a batch of sub-iterations performed by an online convex optimization oracle. More recent work has developed online-to-nonconvex conversions based on stronger or nonstandard notions of regret, including shifting, adaptive, discounted, and dynamic regret \citep{cutkosky2023optimal,ahn2025general,ahn2024adam,patitucci2025improving,zhang2024private}. Although these frameworks have led to powerful convergence guarantees, their dependence on such regret notions does not yield a reduction based solely on ordinary static regret. In particular, dynamic regret can be as large as $\Theta(T)$ \citep{zhao2020dynamic}, in which case the resulting stationarity guarantee may become trivial. By contrast, under standard assumptions in online convex optimization, a static-regret bound of $\cO(T^{1/2})$ is readily available. Moreover, existing conversions are often instantiated through particular update structures or optimization algorithms, such as stochastic gradient and Adam-type methods \citep{cutkosky2023optimal,ahn2025general,ahn2024adam}, rather than treating an arbitrary static-regret minimizer as a black-box. These observations motivate the following natural question, formally posed by \citet{chen2024open}:

\begin{center}
\emph{Can we provide a black-box reduction from offline stochastic nonconvex optimization to static regret minimization in online convex optimization?}
\end{center}
In this paper, we answer this question for static regret against a fixed comparator, and we determine what such a reduction can and cannot achieve. We first develop a black-box reduction from smooth nonconvex optimization to static regret minimization in online convex optimization and then extend this framework to the nonsmooth setting. Given a \(\beta\)-smooth function with initial gap at most \(M\), an OCO oracle $\cA$ in the preconditioner space, and an unbiased stochastic-gradient oracle with variance
bounded by \(\sigma^2\), we establish
$$\frac{1}{T}\sum_{t=1}^T
    \mathbb E\!\left[\|\nabla f(x_t)\|_2^2\right]
    \le
    \frac{101\sigma\sqrt{M\beta}}{\sqrt T}
    +
    \frac{8192\sqrt{M\beta}\,
    \mathscr R_T(\mathcal A,I_d)}{5951T}
    +
    \frac{97M\beta}{T},$$
    where $\mathscr R_T(\mathcal A,I_d)$ is the expected static regret of $\cA$ against the fixed comparator $I_d$. This regret is signed, and the comparator $I_d$ is the identity preconditioner, which defines the identity-preconditioned tracker baseline. The online learner that always plays $I_d$ has zero regret, and our bound then recovers the classical SGD rate \(O(\sigma\sqrt{M\beta/T}+M\beta/T)\) \citep{ghadimi2013stochastic}. A negative regret means that the learned preconditioners outperform the identity on the losses generated by the reduction, and the negative regret term strictly sharpens our stationarity bound. A positive regret is the price of learning the preconditioner. Any black-box OCO algorithm with \(\mathscr R_T(\mathcal A,I_d)=O(\sqrt T)\), such as online gradient descent (Corollary~\ref{cor:ogd}), still recovers the classical \(O(1/\sqrt{T})\) convergence rate.

Our bound differs from the reduction sought by \citet{chen2024open} in one respect. It contains the additional variance term \(\sigma\sqrt{M\beta}/\sqrt T\), whereas in \citet{chen2024open}, the stochastic difficulty enters the guarantee only through the regret.  We show in Section~\ref{sec:lower_bound} that it is unavoidable. Indeed, for every online-to-nonconvex conversion making a constant number of stochastic-oracle queries per round, whose guarantee depends on the online learner only through its static regret against a comparator that the learner can play, the guarantee must contain an additive term of order \(\sigma\sqrt{M\beta/T}\) for sufficiently large \(T\) on instances satisfying our assumptions. Moreover, for our conversion and every online learner \(\cA\), there is an instance that
\[
\frac{1}{T}\sum_{t=1}^T\mathbb E\!\left[\|\nabla f(x_t)\|_2^2\right]\gtrsim\frac{\sigma\sqrt{M\beta}}{\sqrt T}+\frac{\sqrt{M\beta}\,\mathscr R_T(\mathcal A,I_d)}{T}.
\]
Hence our guarantee is tight up to universal constants, including its regret term.

    The key distinction from existing online-to-nonconvex conversions \citep{chen2024open,cutkosky2023optimal} lies in the role assigned to the OCO algorithm. Rather than allowing the online learner to choose the update direction directly, our reduction uses it to select a preconditioner that is applied to a predictable gradient tracker. The resulting preconditioned vector then determines the update at each iteration. This separation enables the OCO component to be treated as a genuine black box whose contribution is characterized solely through its static regret guarantee.
\section{Related Work}
Our work is related to the following streams of research: stochastic nonconvex optimization, online convex optimization, online preconditioning, and online-to-nonconvex-conversion.

\paragraph{Stochastic nonconvex optimization.}The study of first-order methods for stochastic nonconvex optimization dates back to the randomized stochastic-gradient framework of \citet{ghadimi2013stochastic}. For smooth objectives, matching oracle lower bounds were subsequently established by \citet{arjevani2023lower}, building on the zero-chain construction of \citet{carmon2020lower}, and were recently extended to almost surely bounded noise by \citet{jin2026tight}. The hard instances of the oracle lower bound in Section~\ref{sec:lower_bound} additionally have bounded gradients. The nonsmooth setting is more subtle, since approximate Clarke stationarity cannot, in general, be attained in finite time using first-order information. This has motivated the use of relaxed stationarity notions based on the Goldstein subdifferential \citep{goldstein1977optimization}. The finite-time complexity of finding approximate Goldstein stationary points was initiated by \citet{zhang2020complexity} and further developed by \citet{tian2022finite,davis2022gradient,jordan2023deterministic}. 

\paragraph{Online convex optimization.}Online convex optimization studies sequential decision making in which a learner selects an action before observing a convex loss and is evaluated by its regret relative to a comparator \citep{zinkevich2003online,hazan2016introduction}. When the comparator is fixed throughout the horizon, the resulting criterion is known as \emph{static regret}. Canonical algorithms with sublinear static-regret guarantees include online gradient descent \citep{zinkevich2003online}, online mirror descent \citep{srebro2011universality}, and follow-the-regularized-leader \citep{mcmahan2010unified}. Stronger performance criteria have been developed for nonstationary environments. Dynamic regret compares the learner against a time-varying comparator sequence, whose nonstationarity is commonly quantified by its path length \citep{zinkevich2003online,hall2013dynamical}; shifting regret allows the comparator to switch among a collection of experts or decisions \citep{herbster1998tracking}; adaptive and strongly adaptive regret require low regret over every contiguous time interval \citep{hazan2007adaptive,daniely2015strongly}; and discounted-loss formulations reduce the influence of observations from the distant past \citep{chernov2010prediction}. In this paper, our reduction requires only a standard static-regret guarantee. Consequently, any OCO algorithm satisfying the requisite static-regret guarantee can be incorporated into our framework.

\paragraph{Preconditioner optimization and online scaling.}
Preconditioning improves optimization geometry by rescaling gradient directions, and recent work has increasingly treated the preconditioner itself as an optimization variable. Existing approaches include quasiconvex and semidefinite formulations for optimal diagonal scaling \citep{qu2024optimal}, scalable approximations based on matrix–vector products \citep{gao2023scalable,jambulapati2020fast}, and adaptive searches for coordinatewise step sizes using hypergradient information \citep{kunstner2023searching}. More closely related to our work, online scaled gradient methods (OSGM) select matrix-valued step sizes through online learning, using trajectory-dependent surrogate losses to compete with the best fixed preconditioner \citep{gao2025gradient}. This framework has been applied to smooth convex and strongly convex optimization, hypergradient and momentum methods \citep{chu2025provable}, smooth deterministic nonconvex optimization \citep{chu2025practical}, stochastic convex finite-sum problems \citep{zhang2026stochastic}, and online diagonal scaling in quadratic regression dynamics \citep{zhou2026osdn}. Unlike these approaches, our construction uses a generic linear-loss
interface rather than specialized, objective-dependent feedback.

\paragraph{Online-to-nonconvex conversion.}
Early connections between online learning and nonconvex optimization used convex surrogate losses to adapt the stepsize for smooth stochastic objectives \citep{zhuang2019surrogate}. A systematic online-to-nonconvex conversion was developed by \citet{cutkosky2023optimal}, who obtained optimal rates for stochastic nonsmooth nonconvex optimization using online learners with shifting-regret guarantees. Subsequent work instantiated or extended this perspective in several directions. \citet{zhang2024random} connected the conversion with randomly scaled momentum, while \citet{ahn2024understanding} interpreted Adam as discounted FTRL through online learning of updates. \citet{ahn2024adam} established optimal guarantees for clipped Adam with model averaging by exploiting momentum and discounting, and \citet{ahn2025general} developed a broader conversion framework that explains schedule-free SGD through discounted regret. More recent refinements use doubly optimistic feedback under higher-order smoothness assumptions \citep{patitucci2025improving}, or combine discounted-regret conversion with adaptive matrix online learning to derive the Pion and Leon optimizers \citep{jiang2026adaptive}. Despite these advances, existing conversions require additional structure, such as shifting or discounted regret, randomized scaling, model averaging, optimistic hints, smoothness assumptions, or particular online algorithms. They therefore do not provide the black-box static-regret reduction sought by \citet{chen2024open}. In contrast, our reduction accepts an arbitrary OCO algorithm as a black-box oracle, feeds it undiscounted linear losses, and invokes only its static-regret guarantee against a single fixed comparator while achieving optimal stationarity guarantees for stochastic nonconvex objectives.

\section{Notations and Paper Structure}
Throughout the paper, we use
\(
    [n]:=\cbr{1,2,\ldots,n}.\)
For any nonempty closed convex set \(C\subseteq\RR^d\), let
\(\Pi_C\) denote the Euclidean projection onto \(C\).  The notation
\(\lesssim\) suppresses a universal positive multiplicative constant.  We
write \(\cS^{d\times d}\) for the space of real \(d\times d\) symmetric
matrices.  Unless stated otherwise, \(\|\cdot\|\) denotes the Euclidean norm
for vectors, \(\|\cdot\|_{\mathrm F}\) the Frobenius norm for matrices, and
\(\|\cdot\|_{\mathrm{op}}\) the matrix operator norm.
For a random variable \(X\), \(\EE_X\) denotes expectation with respect to
\(X\), and \(\PP_X\) denotes its law. Given two distributions $P$ and $Q$, we use $\text{KL}(P||Q)$ to represent the Kullback-Leibler divergence between them.  For any set \(\Omega\), we write
\(\cP(\Omega)\) for the collection of probability distributions supported
on \(\Omega\). We denote $\textbf{1}$ as the vector with all elements to be $1$.

For \(A\in\RR^{p\times q}\) with a compact singular value decomposition
\(
    A=U\Sigma V^\top,
\)
where the columns of \(U\) and \(V\) are orthonormal and the diagonal entries
of \(\Sigma\) are positive, we define its polar factor by
\(
    \operatorname{polar}(A):=UV^\top,
    \ 
    \operatorname{polar}(0):=0.
\)
Equivalently,
\(
    \operatorname{polar}(A)
    \in
    \operatorname*{argmax}_{\|Z\|_{\mathrm{op}}\le 1}
    \inner{A}{Z}_{\mathrm F}.
\)
For \(X\in\RR^{p\times q}\), \(\operatorname{vec}(X)\in\RR^{pq}\) denotes
the columnwise vectorization of \(X\).  Conversely, for \(d=pq\) and
\(x\in\RR^d\), \(\operatorname{mat}_{p,q}(x)\in\RR^{p\times q}\) denotes
the unique matrix satisfying
\(
    \operatorname{vec}\!\left(\operatorname{mat}_{p,q}(x)\right)=x.
\) 

The remainder of this paper is structured as follows. Section~\ref{sec:setup} formulates the smooth stochastic nonconvex optimization problem, defining the performance criteria and main assumptions. Section~\ref{sec:alg} presents our black-box reduction, and Section~\ref{sec:thm} establishes the corresponding theoretical guarantees. In Section~\ref{sec:extension}, we broaden this framework to accommodate nonsmooth nonconvex optimization. Section~\ref{sec:lower_bound} proves that the variance term in our guarantee cannot be removed by any static-regret conversion and our reduction is minimax optimal. Finally, in Section \ref{sec:num_experiment}, we conduct numerical experiments to illustrate the power of our black-box reduction algorithm.

\section{Problem Setup}\label{sec:setup}
In this section, we introduce the stochastic nonconvex optimization problem
studied in this paper and the online convex optimization oracle that will be
used in the reduction.  The purpose of the setup is to separate the two roles
played by the algorithm.  The outer problem is an offline stochastic
nonconvex optimization problem, where the goal is to find an approximate
stationary point.  The inner oracle is an online convex optimization algorithm,
which will later be used to choose preconditioners through a static regret
guarantee.

We consider the unconstrained nonconvex optimization problem
\begin{equation*}
\tag{P}
\begin{aligned}
\min_{x\in\RR^d}\ f(x).
\end{aligned}
\end{equation*}
Since \(f\) is nonconvex, global optimality is not the appropriate performance
criterion.  Instead, following the standard convention in stochastic nonconvex
optimization, we measure the quality of an algorithm by the average squared
gradient norm along the trajectory or the minimum value of the gradient norm:
\[
    \frac{1}{T}\sum_{t=1}^{T}\|\nabla f(x_t)\|^2,\ \text{or}\ \min_{i\in\cbr{1,\cdots,T}}\|\nabla f(x_i)\|^2.
\]
A small average means that a uniformly sampled iterate from the trajectory
has a small expected squared gradient norm, whereas a small minimum means
that at least one iterate is an approximate stationary point
\citep{zhou2018convergence}.

We impose the following assumptions on the objective function.

\begin{assumption}\label{ass:nonconvex_function}
$f:\RR^d\mapsto\RR$ satisfies the following conditions for some $\beta,M>0$.
\begin{itemize}
    \item $f$ is $\beta$-smooth over $\RR^d$, $\|\nabla f(x)-\nabla f(y)\|_2\le \beta\|x-y\|_2$.
    \item $f(x_1)-\inf_{x\in\RR^d}f(x)\le M$, where $x_1$ is the initial point of the algorithm.
\end{itemize}
\end{assumption}

The smoothness assumption is the standard condition that allows first-order
descent arguments: when the algorithm moves from \(x_t\) to \(x_{t+1}\), the
change in \(f\) can be controlled by a linear term plus a quadratic error.  The
initial-gap condition plays the role of a global progress budget.  It
ensures that the total possible decrease of the objective is finite, which is
what ultimately converts cumulative descent into a stationarity guarantee. We do
not assume that the gradient is bounded. Together with smoothness, the
initial-gap condition controls the gradient at the initial point: by
Lemma~\ref{lemma:gradient_bound}, \(\|\nabla f(x_1)\|_2\le\sqrt{2M\beta}\).
This is the only control on the size of the gradient that our analysis of
smooth objectives uses. If \(f\) has range bounded by \(M\), that is,
\(f(x)-f(y)\le M\) for all \(x,y\), then the initial-gap condition holds.

We access the objective through a stochastic gradient oracle.  This captures
the usual stochastic optimization setting, where the exact gradient
\(\nabla f(x)\) is unavailable, but an unbiased noisy estimate can be sampled.

\begin{assumption}\label{ass:gradient_oracle}
    Let $\cO:\RR^d\mapsto\RR^d$ denote the stochastic gradient oracle. At each query, let $\cH$ denote the complete history before the oracle sample is drawn, including the query point $x$. The oracle outputs $\cO(x)=\tilde{\nabla}f(x)$ such that \[ \EE[\tilde{\nabla}f(x)\mid\cH]=\nabla f(x),\  \EE[\|\tilde{\nabla}f(x)-\nabla f(x)\|_2^2\mid\cH]\le\sigma^2. \]
\end{assumption}

Equivalently, at an iterate \(x_t\), we may write
\[
    \tilde{\nabla}f(x_t)=\nabla f(x_t)+\xi_t,
\]
where \(\xi_t\) is an exogenous zero-mean noise term.  The parameter
\(\sigma^2\) is the variance budget of the stochastic oracle.  In the
reduction, this variance budget determines how aggressively the predictable
gradient tracker can be updated.

We now recall the background of online convex optimization used in the paper.  In
a standard OCO problem, an online player repeatedly chooses decisions from a
convex set before seeing the current loss.  Let \(T\) denote the number of
rounds and let \(\cK_t\) denote the decision set at round $t$.  At round \(t\in[T]\), the
player chooses \(u_t\in\cK_t\).  After this choice is made, a convex loss function
\(\ell_t\in\cF\) is revealed, and the player incurs loss
\(\ell_t(u_t)\). Thus, an OCO algorithm \(\cA\) maps the past history $(\ell_1,\cdots,\ell_{t-1})$
to the current decision \(u_t\in\cK_t\).

There are several notions of regret in online convex optimization.  The most
basic one is static regret, which compares the learner with a single fixed
decision chosen in hindsight.  Given a decision set sequence \(\cbr{\cK_t}_{t=1}^T\) and action sequence $\cbr{u_t}_{t=1}^T$ such that $u_t\in\cK_t, \forall t$, a feasible loss
sequence \((\ell_1,\cdots,\ell_T)\subset\cF\), and a fixed comparator
\(u\), the static regret is
\[
\reg_T(\cA,u):=\sum_{t=1}^T\ell_t(u_t)-\sum_{t=1}^{T}\ell_t(u).
\]
Note that we do not require $u\in\cK_t$, as there might be model mis-specification \citep{ghai2022robust,bennouna2025contextual}, and we may face decision sets that do not contain the comparator $u$.

A stronger  notion is dynamic regret, which compares the
learner with a time-varying comparator sequence
\((\hat{u}_1,\cdots,\hat{u}_T)\):
\[
\dreg_T(\cA,\hat{u}_{1:T}):=\sum_{t=1}^{T}\ell_t(u_t)-\sum_{t=1}^T\ell_t(\hat{u}_t).
\]

Dynamic and adaptive regret provide richer benchmarks than static regret but
are generally harder to control \citep{zhao2020dynamic}.
\citet{chen2024open} reduce offline nonconvex optimization to dynamic regret
minimization, while \citet{cutkosky2023optimal} obtain a related conversion
using \(K\)-shifting regret. These results demonstrate the power of
online-learning methods, but leave open whether ordinary static regret alone
is sufficient.

This paper focuses precisely on that static-regret interface.  We assume
access only to a static regret guarantee about the online convex optimization oracle $\cA$.  The oracle will later be called on specific decision sets and sequence of linear losses
constructed from the stochastic gradients.

\begin{definition}\label{def:static_regret_oracle}
   A static regret minimization oracle $\cA$ in online convex optimization is an algorithm such that for any comparator $u$, there is a number $\mathscr{R}_T(\cA,u)$ such that 
\[
\mathscr R_T(\cA,u)
:=
\EE[\reg_T(\cA,u)]
=
\EE\left[
\sum_{t=1}^T\ell_t(u_t)
-
\sum_{t=1}^T\ell_t(u)
\right].
\]
    Here, $u_1,\cdots,u_T$ are points drawn according to the OCO algorithm $\cA$. The expectation is taken jointly over the randomness of the OCO algorithm, the stochastic-gradient oracle, and the adaptively generated loss sequence. $\mathscr{R}_T(\cA,u)$ depends on the algorithm, the comparator, and the distribution of the loss sequence. A bound on it that holds uniformly over loss sequences in $\cF$, which depends only on the algorithm, the comparator, $\cF$, and the horizon number $T$, can be used in its place in the stationarity upper bounds. The expected regret itself may be positive, zero, or negative.
\end{definition}

At this point, the decision sets \(\cK_t,t=1:T\) and the loss class \(\cF\) have not yet
been specified.  This is intentional.  In the reduction, the OCO decisions will
not be optimization points in \(\RR^d\).  They will be preconditioners, and the
losses will be linear functions measuring how well a chosen preconditioner
aligns the predictable gradient tracker with a fresh stochastic gradient
sample.  These choices are introduced in Section~\ref{sec:alg}.
\section{Algorithm}\label{sec:alg}
In this section, we present our reduction algorithm and explain the main idea underlying its design. At a high level, the proposed algorithm is based on preconditioner optimization. In contrast to existing online-to-nonconvex conversion frameworks \citep{cutkosky2023optimal,ahn2025general,ahn2024adam}, we do not use the online convex optimization oracle to directly determine the update direction of the iterates. Instead, the direction of movement is still driven by gradient-type information, while the OCO algorithm $\cA$ is used to select a preconditioning matrix that modifies the geometry of this direction.

More specifically, at each round, the algorithm maintains an auxiliary vector that summarizes the historical stochastic gradient information and serves as a predictable proxy for the direction along which the current iterate is expected to move. The OCO oracle $\cA$ then chooses a matrix from the prescribed preconditioner class, and this matrix is applied to the auxiliary vector to form the actual update direction. In this sense, the OCO oracle does not control the raw descent direction itself; rather, it controls how the available gradient information is scaled, rotated, or otherwise transformed before being used in the nonconvex update. This distinction is central to our reduction: the nonconvex optimization procedure supplies the stochastic descent signal, while the OCO oracle adaptively learns the preconditioning geometry through static regret minimization. The detailed algorithmic pseudo-code is presented in Algorithm \ref{alg:alg_known_T}.

\begin{algorithm}
\caption{Online-to-Nonconvex Conversion via Preconditioner Optimization}
\label{alg:alg_known_T}
\begin{algorithmic}[1]
\Require Online convex optimization oracle $\cA$ over a compact matrix space $\cU_t$, function $f:\RR^d\mapsto\RR$ satisfying Assumption~\ref{ass:nonconvex_function}, iteration number $T$, stochastic gradient oracle $\cO$, initial point $x_1$.
\State Set $G=\sqrt{M\beta}$, $a=\frac{G}{\max\cbr{G,\sigma\sqrt{T}}}$, $\eta=\frac{a}{64\beta}$, and $m_1=0$.
\For{$t=1,\cdots,T$}
\State Ask $\cA$ for $U_t\in\cU_t$.
\State Draw fresh sample $s_t\sim\cO(x_t)$.
\State Reveal the linear loss function for OCO algorithm $\cA$
\[
L_t(U):=\frac{-1}{G}\inner{s_t}{Um_t}.
\]
\State Update $x_{t+1}=x_{t}-\eta U_tm_t$.
\State Update $m_{t+1}=(1-a)m_t+a s_t$.
\EndFor
\State Output the trajectory $x_1,\cdots,x_T$.
\end{algorithmic}
\end{algorithm}

In Algorithm~\ref{alg:alg_known_T}, at round $t$, the online convex optimization algorithm operates over the matrix class $\cU_t$ that satisfies the following conditions: $\cU_t\subset \cS^{d\times d}$ is compact, convex, and $\|U\|_{op}\le 1,\forall U\in\cU_t$.
For any candidate preconditioner $U\in\cU$, the loss incurred at iteration $t$ is defined as the following linear function:
\[
L_t(U)=\frac{-1}{G}\inner{s_t}{U m_t}.
\]
Since $L_t(\cdot)$ is linear in $U$, and $\cU_t$ is a convex compact subset of the symmetric matrix space, this defines a valid online convex optimization problem. The comparator in the online convex optimization oracle is the identity matrix $I_d$.
Therefore, the OCO algorithm $\cA$ is well-defined over the preconditioner matrix classes $\cbr{\cU_t}_{t=1}^T$.

We now explain the design of Algorithm~\ref{alg:alg_known_T}. The reduction
has three coupled components: a predictable gradient tracker \(m_t\), an online
learner that chooses a preconditioner \(U_t\), and a linear loss that evaluates
the descent effect of the chosen preconditioner.  The parameters are chosen so
that these three components operate on the same scale.

The starting point is the variance budget over the \(T\) oracle calls.  Since
the stochastic oracle has a conditional variance of at most \(\sigma^2\), the total
variance accumulated over \(T\) rounds is at most \(\sigma^2T\). Thus, the
root-variance scale is \(\sigma\sqrt T\).  On the other hand, even when the
oracle is noiseless, the problem has an intrinsic first-order scale
\(G=\sqrt{M\beta}\), arising from the initial gap and smoothness assumptions.
Therefore, the natural budget for the run is
\[
    B=\max\{G,\sigma\sqrt T\}.
\]
This quantity should be interpreted as the total amount of gradient variation
that the tracker and the online regret certificate must be able to absorb
during the entire horizon.

The parameter $a=\frac{G}{B}$
turns this global budget into a local tracking scale.  The vector \(m_t\) is
an exponential moving average of past stochastic gradients. It requires no
bound on the gradients, and the oracle samples may be unbounded, subject only
to the second-moment condition.

The choice of \(a\) expresses the bias-variance tradeoff in tracking
\(\nabla f(x_t)\).  When the variance budget \(B\) is large, \(a\) is small,
so the tracker averages more aggressively and suppresses oracle noise. Conversely, if
the variance budget is small, \(a\) is larger, so the tracker reacts more
quickly to changes in the true gradient. The stepsize $\eta=\frac{a}{64\beta}$
is then tied to the tracking scale.  The target that \(m_t\) tracks is \(\nabla f(x_t)\).  Choosing \(\eta\) of order
\(a/\beta\) ensures that the iterate moves slowly enough such that the target
gradient does not drift faster than the tracker can follow.

Given the predictable direction \(m_t\), the role of the online learner is to
choose how this direction should be preconditioned. The iterate update is
\[
    x_{t+1}=x_t-\eta U_t m_t.
\]
By smoothness, the first-order  Taylor expansion yields 
\[
    f(x_{t+1})-f(x_t)\approx-\eta\langle \nabla f(x_t),U_t m_t\rangle.
\]
Thus, a good preconditioner is one that makes
\(\langle \nabla f(x_t),U_t m_t\rangle\) large, since this corresponds to a
larger decrease in the objective.  The online learner is therefore not asked
to optimize the nonconvex function directly.  It is asked only to choose the
linear operator \(U_t\) that reshapes the already predictable direction
\(m_t\).

Because the true gradient is not observed, the algorithm draws a fresh oracle
sample \(s_t\sim\mathcal O(x_t)\) and uses it to define the linear OCO loss
\[
    L_t(U)=-\frac1G\langle s_t,U m_t\rangle.
\]
This loss is linear in \(U\).  Moreover, since \(U_t\) and
\(m_t\) are fixed before \(s_t\) is sampled,
\[
    \mathbb E[L_t(U_t)\mid\mathcal \cH_t]
    =
    -\frac1G\langle \nabla f(x_t),U_t m_t\rangle.
\]
Therefore \(L_t\) is an unbiased online evaluation of the first-order descent
effect of the chosen preconditioner.  This is the key reason the reduction can
use a standard static-regret guarantee: regret compares the cumulative linear
descent losses of the learned preconditioners against one fixed comparator,
while smoothness converts this comparison back into nonconvex stationarity.

After the loss is revealed to the OCO oracle, the same fresh sample is used to
update the tracker,
\[
    m_{t+1}
    =
    (1-a)m_t+a s_t.
\]
This ordering is important.  The sample \(s_t\) evaluates the preconditioner
chosen before seeing it, and only then becomes part of the history used to form
future predictable directions.  In this way, the algorithm couples stochastic
gradient information with online preconditioner optimization while preserving
the measurability structure needed for the regret and tracking arguments.

\section{Theoretical Guarantee}\label{sec:thm}
We now present the main theoretical guarantee for
Algorithm~\ref{alg:alg_known_T}. The proof and supporting lemmas are
deferred to Appendix~\ref{app:proof_thm}.

Theorem~\ref{thm:known_T} makes precise the quantitative bridge
underlying our reduction. In stochastic nonconvex optimization,
performance is measured by the expected average squared gradient norm
along the trajectory
\citep{ghadimi2013stochastic,ghadimi2016accelerated,
cartis2010complexity,arjevani2023lower}. On the online-learning side,
the relevant quantity is the static comparator regret of the
preconditioner-selection oracle. By placing the online learner in the
preconditioner slot, the theorem converts this regret quantity directly
into a stationarity guarantee, thereby establishing a black-box
connection of the type posed by \citet{chen2024open}. Section~\ref{sec:lower_bound} shows that the additive variance term in this guarantee is necessary, and that the guarantee is tight for every OCO algorithm.

First, for Algorithm \ref{alg:alg_known_T}, we have the following theorem.
\begin{theorem}\label{thm:known_T}
    Assume that Assumptions \ref{ass:nonconvex_function} and \ref{ass:gradient_oracle} hold. If we run Algorithm \ref{alg:alg_known_T} for $T$ iterations, then the expected average gradient norm among the trajectory $x_1,\cdots,x_T$ satisfies
    \[
    \frac{1}{T}\sum_{t=1}^{T}\EE[\|\nabla f(x_t)\|_2^2]\le \frac{101\sigma\sqrt{M\beta}}{\sqrt{T}}+\frac{8192}{5951}\frac{\sqrt{M\beta}\mathscr{R}_T(\cA,I_d)}{T}+\frac{97M\beta}{T}.
    \]
\end{theorem}
Naturally, along this trajectory, we also have
\[
\min_{i\in\cbr{1,2,\cdots,T}}\EE[\|\nabla f(x_i)\|_2^2]\le \frac{101\sigma\sqrt{M\beta}}{\sqrt{T}}+\frac{8192}{5951}\frac{\sqrt{M\beta}\mathscr{R}_T(\cA,I_d)}{T}+\frac{97M\beta}{T}.
\]
Theorem~\ref{thm:known_T} requires no bound on the gradient norm. Its proof controls the size of the gradient only through the initial tracking error $\|m_1-\nabla f(x_1)\|_2^2=\|\nabla f(x_1)\|_2^2\le 2M\beta$, which follows from the initial gap by Lemma~\ref{lemma:gradient_bound}. In particular, Theorem~\ref{thm:known_T} applies to every $\beta$-smooth function with range bounded by $M$. The scale of the online losses matters, however, for the regret of a concrete OCO algorithm. The gradient of $L_t$ with respect to $U$ has Frobenius norm at most $\|s_t\|_2\|m_t\|_2/G$, and worst-case regret bounds require this quantity to be bounded. The next corollary gives such a bound for online gradient descent when the stochastic gradients are bounded almost surely. Its proof is in Appendix~\ref{app:proof_thm}.
\begin{corollary}\label{cor:ogd}
Suppose that Assumptions~\ref{ass:nonconvex_function} and~\ref{ass:gradient_oracle} hold and that $\|\cO(x)\|_2\le S$ almost surely for every $x$. Suppose also that $\cU_t=\cU$ for every $t$, where $I_d\in\cU$ and $D=\max_{U,V\in\cU}\|U-V\|_F$. Let $\cA$ be online gradient descent over $\cU$, which starts from any $U_1\in\cU$ and updates
\[
U_{t+1}=\Pi_{\cU}\Big(U_t+\frac{\gamma}{2G}\big(s_tm_t^\top+m_ts_t^\top\big)\Big),\qquad \gamma=\frac{DG}{S^2\sqrt T},
\]
where $\Pi_{\cU}$ is the Frobenius-norm projection onto $\cU$. Then $\reg_T(\cA,I_d)\le DS^2\sqrt T/\sqrt{M\beta}$ almost surely, and Algorithm~\ref{alg:alg_known_T} with this $\cA$ satisfies
\[
\frac{1}{T}\sum_{t=1}^{T}\EE[\|\nabla f(x_t)\|_2^2]\le \frac{101\sigma\sqrt{M\beta}}{\sqrt{T}}+\frac{8192}{5951}\frac{DS^2}{\sqrt{T}}+\frac{97M\beta}{T}.
\]
\end{corollary}
For instance, $D=\sqrt d$ for $\cU=\cbr{U\in\cS^{d\times d}:0\preceq U\preceq I_d}$.

The identity comparator gives Theorem~\ref{thm:known_T} a baseline-relative
interpretation. For the induced online loss sequence,
\[
\begin{aligned}
\mathscr R_T(\cA,I_d)
=
\EE\left[
\sum_{t=1}^T\bigl(L_t(U_t)-L_t(I_d)\bigr)
\right] =
-\frac{1}{G}\EE\left[
\sum_{t=1}^T
\inner{s_t}{(U_t-I_d)m_t}
\right].
\end{aligned}
\]
Thus, \(\mathscr R_T(\cA,I_d)<0\) means that the learned preconditioners
achieve a larger cumulative gradient-tracker alignment than the identity
preconditioner on the loss sequence generated by the reduction.  The
corresponding negative term strictly sharpens the stationarity certificate
in Theorem~\ref{thm:known_T}.  In contrast, the identity oracle
\(\cA_{\rm Id}\), which selects \(U_t=I_d\) at every round, satisfies $\mathscr R_T(\cA_{\rm Id},I_d)=0$. Consequently, Theorem~\ref{thm:known_T} reduces to
\[
    \frac1T\sum_{t=1}^T
    \EE[\|\nabla f(x_t)\|_2^2]
    \lesssim
    (\frac{\sigma\sqrt{M\beta}}{\sqrt T}
    +
    \frac{M\beta}{T}),
\]
which recovers the convergence rate
of the classical stochastic gradient descent algorithm \citep{ghadimi2013stochastic}. The static regret term in the general theorem, therefore, has a precise meaning: it measures the effect of learning an
adaptive preconditioner on the induced online losses relative to the
identity preconditioner. Specifically,
\[
\operatorname{Reg}_T(\mathcal A;I_d)
\begin{cases}
=0, & \text{identity-preconditioned baseline},\\
>0, & \text{the learned preconditioner pays an adaptation cost},\\
<0, & \text{the learned preconditioner outperforms identity on the induced losses}.
\end{cases}
\] 
On the other hand, the term $\frac{\sigma\sqrt{M\beta}}{\sqrt T}$
is the usual stochastic nonconvex optimization term. It is the cost of using
a noisy stochastic gradient oracle and matches the classical square-root rate
for finding an approximate stationary point \citep{ghadimi2013stochastic}, which is generally unavoidable due to the stochastic error of gradient sampling \citep{arjevani2023lower}. In Section~\ref{sec:lower_bound}, we show that this term cannot be absorbed into the regret either: every conversion making a constant number of stochastic-oracle queries per round, whose guarantee is stated through static regret against a comparator that the online learner can play, must contain an additive term of this order for sufficiently large $T$.

Moreover, adversarially speaking, even if we only know that  $\cA$ has a worst-case static regret of order
$C_{\cA}T^{1/2}$, our reduction procedure guarantees a classical
$O(\frac{\sqrt{M\beta}(\sigma+C_\cA)}{\sqrt{T}})$ convergence rate, which is
standard in nonconvex optimization
\citep{cartis2010complexity,arjevani2023lower}. Corollary~\ref{cor:ogd} is an example, with $C_\cA=DS^2/\sqrt{M\beta}$. Thus, we show that any online convex optimization algorithm with $O(\sqrt{T})$ static regret is sufficient for obtaining this nonconvex stationarity guarantee.

This distinction is important for adaptive methods.  Existing convergence
proofs for Adam, AdaGrad, and related methods often analyze the optimization algorithms directly \citep{kingma2014adam,defossezsimple}.  In contrast, our
result uses the online algorithm only through \(\mathscr R_T(\cA,I_d)\).  If
\(\cA\) is an AdaGrad-type or matrix-adaptive online learner whose regret bound
depends on a data-dependent variation term, then the same term is inherited by
the nonconvex guarantee, as discussed in \citet{chen2024open}. Therefore, we believe that our reduction framework has a deep connection with modern optimizers in deep learning. Determining whether additional practical optimizers admit the required preconditioner-space regret interface is an interesting direction for future work.

\begin{remark}
    If we have an exact gradient oracle, which means that we can query the exact $\nabla f(x)$ at any point $x$ instead of using a stochastic estimator, then we have that $\sigma=0$ and we have 
\[
\frac{1}{T}\sum_{t=1}^{T}\EE[\|\nabla f(x_t)\|_2^2]\lesssim\frac{\sqrt{M\beta}\mathscr{R}_T(\cA,I_d)}{T}+\frac{M\beta}{T},
\]
where the leading term completely depends on our static regret minimization oracle $\cA$.
\end{remark}
Finally, unlike dynamic regret, which can be \(\Omega(T)\) in adversarial settings and thus fail to imply convergence, when the preconditioner class and the stochastic gradients are bounded,
standard online linear-optimization algorithms admit
\(O(\sqrt{T})\) worst-case static-regret guarantees, which is sublinear in \(T\) (see Corollary~\ref{cor:ogd}). Therefore, our reduction always yields a meaningful convergence rate.

We now provide a proof sketch of Theorem \ref{thm:known_T}. The full proof is in Appendix \ref{app:proof_thm}.
\paragraph{Proof sketch of Theorem \ref{thm:known_T}}
Write
\[
    g_t=\nabla f(x_t),
    \ 
    S=\sum_{t=1}^T \EE[\|g_t\|^2],
    \ 
    e_t=m_t-g_t,
    \ 
    E_{\rm tr}=\sum_{t=1}^T\EE[\|e_t\|^2] .
\]
First, smoothness converts the progress of the nonconvex objective into a lower bound on the cumulative OCO loss of the played preconditioners. Since
\(x_{t+1}=x_t-\eta U_tm_t\) and \(\|U_t\|_{\rm op}\le 1\), by the smoothness of the nonconvex function $f$, we have
\[
    f(x_{t+1})-f(x_t)
    \le
    -\eta\langle g_t,U_tm_t\rangle
    +
    \frac{\beta\eta^2}{2}\|m_t\|^2 .
\]
Taking the conditional expectation and using
\(\EE[s_t\mid \mathcal F_t]=g_t\), this becomes
\[
    \EE[f(x_{t+1})-f(x_t)]
    \le
    \eta G\,\EE [L_t(U_t)]
    +
    \frac{\beta\eta^2}{2}\EE[\|m_t\|^2] .
\]

Taking total expectation and summing over \(t=1,\ldots,T\),
\[
    \EE[f(x_{T+1})-f(x_1)]
    \le
    \eta G\sum_{t=1}^T\EE [L_t(U_t)]
    +
    \frac{\beta\eta^2}{2}
    \sum_{t=1}^T\EE[\|m_t\|^2] .
\]
By Assumption~\ref{ass:nonconvex_function}, $f(x_1)-f(x_{T+1})\le M$, and hence
\[
    -M
    \le
    \eta G\sum_{t=1}^T\EE [L_t(U_t)]
    +
    \frac{\beta\eta^2}{2}
    \sum_{t=1}^T\EE[\|m_t\|^2]\Rightarrow G\sum_{t=1}^T\EE [L_t(U_t)]
    \ge
    -\frac{M}{\eta}
    -
    \frac{\beta\eta}{2}
    \sum_{t=1}^T\EE[\|m_t\|^2] .
\]
Finally, since \(\eta=a/(64\beta)\) and \(G^2=M\beta\), we have $\frac{M}{\eta}
    =
    \frac{64M\beta}{a}
    =
    64\frac{G^2}{a},
    \ 
    \frac{\beta\eta}{2}
    =
    \frac a{128}$.
    
Plugging this in, we obtain 
\begin{align}\label{equa:equa1}
G\sum_{t=1}^T\EE [L_t(U_t)]
    \gtrsim
    -\frac{G^2}{a}
    -a\sum_{t=1}^{T}\EE[\|m_t\|^2].
    \end{align}
Second, applying the static regret oracle bound with the
fixed comparator \(I_d\), we have:
\[
    \sum_{t=1}^T L_t(U_t)
    \le
    \sum_{t=1}^T L_t(I_d)+\mathscr \reg_T(\mathcal A,I_d).
\]
Multiplying by \(G\) and taking expectations gives
\[
    G\sum_{t=1}^T\EE [L_t(U_t)]
    \le
    G\sum_{t=1}^T\EE [L_t(I_d)]
    +
    G\,\EE\mathscr [\reg_T(\mathcal A,I_d)].
\]
Notice that $G\,\EE [L_t(I_d)]
    =
    -\EE[\langle g_t,m_t\rangle]
    =
    -\EE[\|g_t\|^2]-\EE[\langle g_t,e_t\rangle].$
By the Cauchy-Schwarz inequality, we have $G\sum_{t=1}^T\EE [L_t(I_d)]
    \le
    -S+\sqrt{S E_{\rm tr}}.$
Thus, we have 
\begin{align}\label{equa:equa2}
G\sum_{t=1}^T\EE [L_t(U_t)]\lesssim -S+\sqrt{SE_{tr}}+G\EE[\reg_T(\cA,I_d)]. 
\end{align}
Third, the exponential tracker is accurate enough.  The update
\[
    m_{t+1}=(1-a)m_t+as_t
\]
balances two errors: the stochastic noise in \(s_t\), whose cumulative budget
is \(V_T=\sum_{t=1}^T\EE[\|s_t-g_t\|^2]\), and the drift of the true gradient
between \(x_t\) and \(x_{t+1}\).  Smoothness and
\(\eta=a/(64\beta)\) make this drift smaller than a constant multiple of
\(a\|m_t\|\).  The resulting tracker estimate is
\begin{align}\label{equa:equa3}
    E_{\rm tr}
    \lesssim
    \frac{G^2}{a}
    +
    aV_T
    +
    cS
\end{align}
for a sufficiently small universal constant \(1>c>0\).

Combining Inequality (\ref{equa:equa1}), (\ref{equa:equa2}), and (\ref{equa:equa3}), we have
\[
    S
    \lesssim
    \frac{G^2}{a}
    +
    aV_T
    +
    G\,\EE[\reg_T(\mathcal A,I_d)].
\]
With the choices $B=\max\{G,\sigma\sqrt T\},
    \ 
    a=\frac{G}{B},$
and \(V_T\le \sigma^2T\le B^2\), we have
\[
    \frac{G^2}{a}=GB,
    \ 
    aV_T\le GB.
\]
Thus
\[
    S
    \lesssim
    G B
    +
    G\,\EE[\reg_T(\mathcal A,I_d)].
\]
Dividing by \(T\) and using \(G=\sqrt{M\beta}\), $\EE[\reg_T(\cA,I_d)]=\mathscr{R}_T(\cA,I_d)$ gives
\[
    \frac1T\sum_{t=1}^T\EE[\|\nabla f(x_t)\|^2]
    \lesssim
    \sigma\sqrt{\frac{M\beta}{T}}
    +
    \frac{M\beta}{T}
    +
    \frac{\sqrt{M\beta}\,\mathscr R_T(\mathcal A,I_d)}{T}.
\]
Thus, we finish the proof.

\section{From Nonconvex Nonsmooth Optimization to Static Regret Minimization}\label{sec:extension}
In this section, we consider a more ambitious challenge. Instead of smooth functions in Section \ref{sec:setup}, we focus on a black-box reduction from nonsmooth nonconvex optimization to static regret minimization in online convex optimization. Specifically, we consider the function $f(x)$ that is differentiable, whose gradients need not be Lipschitz continuous. Instead, we only assume that the gradient is bounded for some number $L$ and that the function satisfies the Newton-Leibniz formula.
\begin{assumption}\label{ass:nonsmooth}
    The function $f$ satisfies $\|\nabla f(x)\|\le L$ for every $x$ and
$f(x_1)-\inf_{x\in\RR^d}f(x)\le M$ for some $L,M>0$,
where $x_1$ is the initial point. Also, the Newton-Leibniz formula holds for $f$, i.e., $f$ is absolutely continuous,
    \[
    f(x+v)-f(x)=\int_{0}^1\inner{\nabla f(x+rv)}{v}dr.
    \]
\end{assumption}
Technically, the main difficulty of nonsmooth optimization is that the gradient descent lemma is no longer available. Therefore, at every iteration, we have to sample uniformly along the direction in which we are moving this round. 

Assumption \ref{ass:gradient_oracle}remains in force throughout this section and we omit them here. Our performance metric in nonsmooth nonconvex optimization is the $(\lambda,\varepsilon)-$stationary point, which is a common performance measure in nonsmooth optimization \citep{ahn2025general} that can be traced back to the Goldstein stationary point \citep{goldstein1977optimization}.
\begin{definition}
    Suppose $f:\RR^d\mapsto\RR$ is differentiable. We say $x$ is a $(\lambda,\varepsilon)$-stationary point if $\|\nabla f(x)\|^{[\lambda]}\le \varepsilon$, where
    \[
    \|\nabla f(x)\|^{[\lambda]}=\inf_{p\in\cP(\RR^d),\EE_{w\sim p}[w]=x}\cbr{\|\EE[\nabla f(w)]\|+\lambda\EE[\|x-w\|^2]}.
    \]
\end{definition}
In this section, we provide an algorithm that reduces nonsmooth nonconvex optimization to static regret minimization in online convex optimization, which significantly extends our previous reduction framework.

Algorithm \ref{alg:nonsmooth_static_reduction} extends the preconditioner-based
reduction in Algorithm~\ref{alg:alg_known_T} to objectives without
Lipschitz-continuous gradients.  The central online-learning architecture is
unchanged: a single black-box OCO oracle is run throughout all \(T\) rounds,
and, before observing the current stochastic gradient, it selects a matrix
\(U_t\in\cU_t\) that preconditions a direction determined entirely by the past. As in Section \ref{sec:alg}, we assume that $\cU_t$ is compact and that $\|U\|_{op}\le 1, \forall U\in\cU_t$.

After observing the fresh sample \(s_t\), the oracle receives the undiscounted
linear loss
\[
    L_t(U)=-\langle s_t,Uz_t\rangle.
\]
Consequently, the entire loss sequence is governed by one ordinary static-regret
guarantee against a fixed matrix comparator; neither the oracle nor its regret
sequence is restarted or discounted.

The principal difference from the smooth reduction lies in how one iteration
is connected to the change in the objective value.  In
Algorithm~\ref{alg:alg_known_T}, the stochastic gradient is queried directly at
\(x_t\), and smoothness supplies a quadratic upper bound on
\(f(x_{t+1})-f(x_t)\).  Such a bound is unavailable in the present setting.
Algorithm~\ref{alg:nonsmooth_static_reduction} therefore first commits to the
candidate displacement
\[
    \delta_t=-\eta U_tz_t
\]
and then queries the stochastic-gradient oracle at a uniformly random point
\(y_t=x_t+r_t\delta_t\) on the segment joining \(x_t\) and \(x_{t+1}\). This
random segment query replaces the smooth descent lemma by the exact conditional
identity $\EE_{r_t}
    \bigl[
        \langle\nabla f(x_t+r_t\delta_t),\delta_t\rangle
    \bigr]
    =
    f(x_{t+1})-f(x_t).$
Conditional unbiasedness of \(s_t\) makes the online loss an unbiased
measurement of this objective change.

A second difference is that the tracker \(m_t\) is replaced in the update
by the smoothly normalized direction
\[
    z_t
    =
    \frac{bm_t}{\sqrt{b^2\|m_t\|_2^2+\rho^2}}.
\]
This normalization guarantees \(\|z_t\|_2\le1\), keeps every displacement
bounded by \(\eta\), and allows for the loss of the identity comparator to control
the magnitude of the predictable tracker up to the regularization error
introduced by \(\rho\).  Finally, whereas the smooth reduction controls the
average gradient norm along the original trajectory, the nonsmooth reduction
constructs the weighted points \(\bar y_t\) and returns a randomized
\(\bar y_\tau\).  These weights are used only to form a local stationarity
certificate; importantly, they do not modify the undiscounted OCO losses.

\begin{algorithm}
\caption{Online-to-Nonconvex Conversion via Preconditioner Optimization for Nonsmooth Function}
\label{alg:nonsmooth_static_reduction}
\begin{algorithmic}[1]
\Require Initial point \(x_1\), total iteration number \(T\), parameters
\(a\in(0,1/2]\), \(\eta>0\), and \(\rho>0\), stochastic-gradient oracle
\(\cO\), and one static-regret oracle \(\cA\) over $\cbr{\cU_t}_{t=1}^T$
\State Set $b=1-a,\ 
    \BB_L=\{v\in\RR^d:\|v\|_2\le L\},\ 
    m_1=0,\ c_0=0,\ \bar y_0=0$.
\For{\(t=1,\ldots,T\)}
    \State Form the normalized predictable direction
    \[
        z_t
        =
        \frac{bm_t}
        {\sqrt{b^2\|m_t\|_2^2+\rho^2}}.
    \]
    \State Ask the static-regret oracle \(\cA\) for \(U_t\in\cU_t\).
    \State Set the displacement and its candidate endpoint as $\delta_t=-\eta U_tz_t,
        \ 
        x_{t+1}=x_t+\delta_t.$
    \State Draw \(r_t\sim\mathrm{Unif}[0,1]\) and set
    $y_t=x_t+r_t\delta_t.$
    \State Draw a fresh stochastic-gradient sample \(s_t\sim\cO(y_t)\).
    \State Reveal to \(\cA\) the linear loss
    \[
        L_t(U)=-\langle s_t,Uz_t\rangle.
    \]
    \State Update the projected predictable tracker
    \[
        m_{t+1}
        =
        \Pi_{\BB_L}\bigl(bm_t+as_t\bigr).
    \]
    \State Set \(c_t=1-b^t\) and update the weighted output point
    \[
        \bar y_t
        =
        \frac{bc_{t-1}}{c_t}\bar y_{t-1}
        +
        \frac{a}{c_t}y_t.
    \]
\EndFor
\State Draw \(\tau\in\{1,\ldots,T\}\) according to
\[
    \Pr(\tau=t)
    =
    \begin{cases}
        c_t/T, & 1\le t<T,\\[1mm]
        c_T/(aT), & t=T,
    \end{cases}
\]
and output \(\bar y_\tau\).
\end{algorithmic}
\end{algorithm}
With Algorithm \ref{alg:nonsmooth_static_reduction}, we have the following guarantee, which yields an optimal asymptotic rate in $T$ for fixed $M,L,\sigma,\lambda>0$, as long as our static regret minimization oracle enjoys an $O(T^{1/2})$ regret bound.
\begin{theorem}\label{thm:nonsmooth_known_T}
    If we run Algorithm \ref{alg:nonsmooth_static_reduction} for $T$ rounds, then for any $a\in(0,1/2]$, $\eta>0$, $\rho>0$, and $\lambda>0$,
    \[
    \EE[\|\nabla f(\bar y_\tau)\|^{[\lambda]}]\le \frac{2M}{\eta T}+\frac{2\mathscr{R}_T(\cA,I_d)}{T}+\rho+\frac{a(L^2+\sigma^2)}{\rho}+\sigma\sqrt{\frac{a}{2-a}}+\frac{L}{aT}+(2+\frac{4(1-a)}{a^2})\lambda\eta^2.
    \]
    Specifically, let $B=\sqrt{L^2+\sigma^2}$ and define the scale $\varepsilon_T=(\frac{B^4M^2\lambda}{T^2})^{1/7}$. If we set $a=\frac{\varepsilon_T^2}{B^2}$, $\eta=\frac{\varepsilon_T^{5/2}}{B^2\sqrt{\lambda}}$, and $\rho=\varepsilon_T$, then when $T\ge \frac{2^{7/4}M\sqrt{\lambda}}{B^{3/2}}$, we have
    \[
    \EE[\|\nabla f(\bar y_\tau)\|^{[\lambda]}]\le11\varepsilon_T+\frac{L}{M\sqrt{\lambda}}\varepsilon_T^{3/2}+\frac{2\mathscr{R}_T(\cA,I_d)}{T}.
    \]
\end{theorem}
Therefore, for fixed $M,L,\sigma,\lambda>0$, when $\cA$ has an
$O(T^{1/2})$ static regret bound, we achieve the optimal asymptotic rate
$O(T^{-2/7})$ for the relaxed stationarity measure
\citep{zhang2024random}. The full proof of this theorem is deferred to Appendix \ref{app:proof_extension}.

\section{Lower Bounds: The Variance Term Is Unavoidable for Static-Regret Conversions}\label{sec:lower_bound}
Theorem~\ref{thm:known_T} contains the additive term $\frac{101\sigma\sqrt{M\beta}}{\sqrt T}$, which remains even when the regret vanishes. For instance, the identity oracle $\cA_{\rm Id}$ satisfies $\mathscr R_T(\cA_{\rm Id},I_d)=0$. In this section, we show that this term is not an artifact of our analysis.
 
First, consider any online-to-nonconvex conversion that makes a constant number of oracle queries per round and whose guarantee is stated through the static regret $\mathscr R_T(\cA,u)$. Either the guarantee carries an additive term of order $\sigma\sqrt{M\beta/T}$, or, when the coefficient of the regret is of order $\sqrt{M\beta}$ as in Theorem~\ref{thm:known_T}, every OCO algorithm incurs regret of order $\sigma\sqrt T$ on some instance. When the comparator $u$ can be played by the OCO algorithm, as $I_d$ in Theorem~\ref{thm:known_T} whenever $I_d\in\cU_t$, the second alternative is impossible. Hence, the variance term in Theorem~\ref{thm:known_T} is optimal up to universal constants. Second, for Algorithm~\ref{alg:alg_known_T}, the whole bound of Theorem~\ref{thm:known_T}, including its regret term, is tight up to a universal constant for every OCO algorithm. Both results rely on an oracle lower bound for stochastic first-order methods, which we prove in Appendix~\ref{app:proof_lower_bound}.
 
To state a lower bound that applies to every reduction design, we first formalize the class of conversions. An online-to-nonconvex conversion $\cC$ chooses the decision sets $\cK_1,\ldots,\cK_T$ and the losses revealed to an OCO algorithm $\cA$, and it interacts with the stochastic gradient oracle $\cO$ and with $\cA$ as described in Protocol~1. Its design is otherwise unrestricted.
 
\begin{center}
\fbox{\begin{minipage}{0.94\textwidth}
\textbf{Protocol 1} (conversion $\cC$ interacting with an OCO algorithm $\cA$). For $t=1,\ldots,T$:
\begin{enumerate}
    \item $\cA$ outputs $u_t\in\cK_t$.
    \item $\cC$ makes at most $q$ queries to $\cO$. Each query point may depend on everything $\cC$ has observed so far, on $u_1,\ldots,u_t$, and on its own random bits.
    \item $\cC$ reports an iterate $x_t$, which may be any function of what it has observed.
    \item $\cC$ reveals to $\cA$ a convex loss $\ell_t\in\cF$, which may be any function of what it has observed.
\end{enumerate}
\end{minipage}}
\end{center}
 
Stationarity is measured at the reported iterates through $\frac{1}{T}\sum_{t=1}^{T}\EE[\|\nabla f(x_t)\|_2^2]$, as in Theorem~\ref{thm:known_T}. For example, Algorithm~\ref{alg:alg_known_T} is a conversion with $q=1$, $\cK_t=\cU_t$, $u_t=U_t$, and $\ell_t=L_t$, whose reported iterates are its query points $x_t$. We call a pair $(f,\cO)$ satisfying Assumptions~\ref{ass:nonconvex_function} and Assumption~\ref{ass:gradient_oracle} an instance, where $x_1$ in Assumption~\ref{ass:nonconvex_function} is the initial point given to the conversion. This is exactly the setting of Theorem~\ref{thm:known_T}. These are exactly the assumptions of \citet{chen2024open}, namely $\beta$-smoothness, $f(x_1)-\inf f\le M$, and an unbiased oracle with variance at most $\sigma^2$, so the results below apply to every conversion designed for their setting.
 
We evaluate the OCO algorithm through the static regret of Section~\ref{sec:setup}. For a fixed comparator $u$, which need not belong to the decision sets $\cK_t$,
\[
\reg_T(\cA,u)=\sum_{t=1}^T\ell_t(u_t)-\sum_{t=1}^{T}\ell_t(u),\qquad \mathscr R_T(\cA,u)=\EE[\reg_T(\cA,u)],
\]
where the expectation is taken over the randomness of $\cA$, $\cC$, and $\cO$ in the instance under consideration. Then, we can rigorously define the class of online-to-nonconvex conversion with static regret-controlled guarantee, which is the conversion proposed by \citet{chen2024open}
\begin{definition}\label{def:regret_controlled}
A conversion $\cC$ is said to have a static regret-controlled guarantee with comparator $u$, additive term $b_T$, and coefficient $\alpha_T$, if
\begin{align}\label{eq:regret_controlled}
\frac{1}{T}\sum_{t=1}^{T}\EE[\|\nabla f(x_t)\|_2^2]\le b_T+\alpha_T\frac{\mathscr R_T(\cA,u)}{T}
\end{align}
for every OCO algorithm $\cA$ and every instance $(f,\cO)$. Here $b_T\ge 0$ and $\alpha_T\ge 0$ may depend on $(M,\beta,\sigma,T,q)$, but not on the instance, the dimension, or $\cA$. The guarantee is variance-free if $b_T\le CM\beta/T$ and $\alpha_T\le C\sqrt{M\beta}$ for a universal constant $C$.
\end{definition}
 
Theorem~\ref{thm:known_T} has exactly this form. A variance-free guarantee is the fixed-comparator analog of the reduction sought by \citet{chen2024open}, in which the stochastic difficulty enters the stationarity guarantee only through the regret of $\cA$.
 

Composing a conversion with any fixed OCO algorithm produces a stochastic first-order method. Our next theorem shows that for all possible online-to-nonconvex conversions, the variance in Theorem \ref{thm:known_T} cannot be reduced and has to be paid either through $b_T$ or through the regret.

This yields the following dichotomy, whose proof is deferred to Appendix~\ref{app:proof_lower_bound}.
\begin{theorem}\label{thm:dichotomy}
Let $\cC$ be a conversion with a static regret-controlled guarantee according to Definition~\ref{def:regret_controlled} with comparator $u$, additive term $b_T$, and coefficient $\alpha_T$. If $(q+1)T\ge 16\max\cbr{\frac{M\beta}{\sigma^2},\frac{\sigma^2}{M\beta}}$, then there exists a universal constant $c>0$ such that at least one of the following statements holds.
\begin{itemize}
    \item[(i)] The additive term carries the variance:
    \[
    b_T\ge c\,\sigma\frac{\sqrt{M\beta}}{\sqrt{(q+1)T}}.
    \]
    \item[(ii)] $\alpha_T>0$, and for every OCO algorithm $\cA$, there exists an instance on which
    \[
    \mathscr R_T(\cA,u)\ge c\,\frac{\sqrt{M\beta}}{\alpha_T}\,\sigma\sqrt{\frac{T}{q+1}}.
    \]
\end{itemize}
In particular, if $\alpha_T\le C\sqrt{M\beta}$, as for the coefficient in Theorem~\ref{thm:known_T}, then statement (ii) says that every OCO algorithm incurs regret of order $\sigma\sqrt{T/(q+1)}$ on some instance.
\end{theorem}
Theorem~8.2 shows that the stochastic difficulty must be reflected
either in the additive term or in the regret for some instances.
When the comparator is playable, the constant-comparator strategy
has zero regret, so alternative~(ii) is impossible and the additive
variance term is unavoidable.

Finally, we prove that the bound in our Theorem~\ref{thm:known_T} is tight and optimal.
\begin{theorem}\label{thm:tightness}
Let $T\ge 16\max\cbr{\frac{M\beta}{\sigma^2},\frac{\sigma^2}{M\beta}}$. For every OCO algorithm $\cA$ in Algorithm~\ref{alg:alg_known_T}, there exists an instance $(f,\cO)$ such that
\[
\frac{1}{T}\sum_{t=1}^{T}\EE[\|\nabla f(x_t)\|_2^2]\ge c\left(\frac{\sigma\sqrt{M\beta}}{\sqrt T}+\frac{\sqrt{M\beta}\,\mathscr R_T(\cA,I_d)}{T}\right),
\]
where $c$ is a universal constant.
\end{theorem}
The proof is deferred to Appendix~\ref{app:proof_lower_bound}. The essential block for the construction of the lower bound is in Lemma~\ref{lemma:oracle_lower_bound}. It combines the techniques from \citet{carmon2020lower}, \citet{jin2026tight}, and \citet{arjevani2023lower} to build the hard instance.

\section{Numerical Experiments}\label{sec:num_experiment}
 
We compare the black-box reduction instantiated with a two-expert Hedge oracle
against stochastic gradient descent (SGD) on four globally bounded, smooth,
nonconvex objectives $f_1,f_2,f_3,f_4$.  In every experiment,
\[
f_i(W_1)-f_i^\star=0.5,
\  i\in\{1,2,3,4\},
\]
so all four trajectories begin with the same nonzero optimization error.
 
\textbf{1: We first discuss the form of $f_1$.}
Let $W=(z,x,y)\in\mathbb R^{129}$, where $x,y\in\mathbb R^{64}$, and define
\[
r_j=\frac{x_j+y_j}{\sqrt2},
\ 
s_j=\frac{-x_j+y_j}{\sqrt2}.
\]
\begin{equation}
f_1(W)
=1-\cos z
+8\sum_{j=1}^{64}
\left[
1-\cos r_j+0.1(1-\cos s_j)
\right].
\label{eq:f1}
\end{equation}
At $z=\pi$, $\partial_{zz}^2f_1(W)=-1<0$, so $f_1$ is nonconvex.
 
\textbf{2: Now, for the second function $f_2$}, we
define
\[
D(s)=1-\exp\!\left(-\frac{(s^2-1)^2}{2}\right),
\ 
\psi(s)=1-\exp\!\left(-\frac{s^2}{2}\right),
\]
and let
\[
q(s)=D(s)-0.1\tanh s,
\ 
g_\star=\min_s q(s)=-0.0763742876.
\]
For $W=(c,u,v,z)\in\mathbb R^{67}$, where $z\in\mathbb R^{64}$, set
\begin{equation}
f_2(W)
=D(c)
+s_c\left[
q\!\left(\frac{u}{\ell}\right)-g_\star
+4\psi\!\left(\frac{v-0.3u}{\ell}\right)
\right]
+40\sum_{i=1}^{64}\psi(z_i),
\label{eq:f2}
\end{equation}
with
\[
\ell=0.275\sqrt{0.02}=0.0388908730,
\ 
s_c=0.0021279855.
\]
Since $\partial_{cc}^2f_2(W)=D''(c)$ and $D''(0)=-2e^{-1/2}<0$, $f_2$ is
nonconvex.
 
\textbf{3: Regarding the third function $f_3$}, for $W=\theta\in\mathbb R^{64}$ with $\theta_{65}=\theta_1$, we define
\begin{equation}
f_3(\theta)
=\sum_{j=1}^{64}\left[1-\cos(\theta_{j+1}-\theta_j)\right]
+32\sum_{j=1}^{64}\left[1-\cos\theta_j\right].
\label{eq:f3}
\end{equation}
At $\theta=\pi\mathbf 1$, the Hessian matrix of $f_3$ is $L_{C_{64}}-32I$, so $f_3$ is nonconvex.
 
\textbf{4: Finally, we set the following function to be $f_4$.}
For $W=y\in\mathbb R^{64}$, define
\begin{equation}
f_4(y)
=1-\exp\!\left(-\frac{\|y\|_2^2}{2}\right)
+32\sum_{i=1}^{64}(1-\cos y_i).
\label{eq:f4}
\end{equation}
At $y=\pi e_1$,
$e_1^\top\nabla^2f_4(y)e_1=-32+(1-\pi^2)e^{-\pi^2/2}<0$, so $f_4$ is
nonconvex.
 
\paragraph{Initialization.}
The fixed initial points, initial gaps, and exact gradient norms are listed in
Table~\ref{tab:nonconvex_initialization}.
 
\begin{table}[H]
\centering
\small
\renewcommand{\arraystretch}{1.13}
\begin{tabular}{l l c c}
\hline
Function & Initial point & $f_i(W_1)-f_i^\star$
& $\|\nabla f_i(W_1)\|_2$ \\
\hline
$f_1$
& $z=\pi/3,\ x=y=0$
& $0.5000$ & $0.8660$ \\
$f_2$
& $(c,u,v,z)=(1.4744540,0,-0.0194454,0)$
& $0.5000$ & $1.7408$ \\
$f_3$
& $\theta_j=0.0312463\sin(2\pi j/64)$
& $0.5000$ & $5.6572$ \\
$f_4$
& $y=0.1743116e_1$
& $0.5000$ & $5.7214$ \\
\hline
\end{tabular}
\caption{Initializations used for $f_1,f_2,f_3,f_4$.}
\label{tab:nonconvex_initialization}
\end{table}
 
\paragraph{Stochastic oracle and Hedge experts.}
For each $f_i$, Hedge mixes $I$ with one fixed orthogonal projector $P_i$.
The stochastic oracle is $s_t=\nabla f_i(W_t)+\xi_t$, where $P_i\xi_t=0$:
\begin{itemize}
    \item for $f_1$, $P_1$ keeps $z$ and every $y_j$ but
    removes every $x_j$; the 64 coordinates of $x$ receive independent
    $0.1$-Rademacher noise;
    \item for $f_2$, $P_2$ keeps $c$, projects
    $(u,v)$ onto $\operatorname{span}\{(1,1)\}$, and removes $z$; the 64
    coordinates of $z$ receive independent $0.4$-Rademacher noise;
    \item for $f_3$, $P_3$ projects onto the constant and first
    sine/cosine Fourier modes, and
    $\xi_t=0.4(I-P_3)\epsilon_t$ for an i.i.d. Rademacher vector $\epsilon_t$;
    \item for $f_4$, $P_4=e_1e_1^\top$, while coordinates
    $2,\ldots,64$ receive independent $0.4$-Rademacher noise.
\end{itemize}
All four projectors are symmetric positive semidefinite contractions, and all
four stochastic oracles are unbiased and almost surely bounded.
 
Starting from $p_{1,I}=p_{1,P_i}=1/2$, Hedge algorithm plays
$U_t=p_{t,I}I+p_{t,P_i}P_i$ and receives
\[
\ell_t(U)=-\frac{1}{G}\langle s_t,Um_t\rangle,
\ 
p_{t+1,U}\propto p_{t,U}\exp[-\gamma\ell_t(U)].
\]
The black-box update uses the old tracker and the old Hedge weights,
\[
W_{t+1}=W_t-\eta U_tm_t,
\ 
m_{t+1}=(1-a)m_t+a s_t,
\  m_1=0.
\]
For a detailed introduction to the Hedge algorithm, we refer the reader to Appendix \ref{app:math_tools}. 
The SGD baseline is $W_{t+1}=W_t-\eta s_t$.
 
\paragraph{Hyperparameter setting.}
The declared horizon is $T_{\rm run}=20000$.  For each objective, we use
\[
G=\sqrt{M\beta},
\ 
a=\frac{G}{\max\{G,\sigma\sqrt{T_{\rm run}}\}},
\ 
\eta=\frac{a}{64\beta},
\]
as in Algorithm~\ref{alg:alg_known_T}. Let $S$ be the sum of an upper bound on
$\sup_W\|\nabla f_i(W)\|_2$ and the maximum of $\|\xi_t\|_2$, so that
$\|s_t\|_2\le S$ almost surely. Since $m_1=0$ and $m_{t+1}=(1-a)m_t+as_t$, each
$m_t$ is a convex combination of $0,s_1,\ldots,s_{t-1}$, so $\|m_t\|_2\le S$.
Hence both expert losses lie in $[-C,C]$, and we set
\(
C=\frac{S^2}{G},
\ 
\gamma=\sqrt{\frac{2\log 2}{T_{\rm run}C^2}}.
\)
 
The resulting constants appear in
Table~\ref{tab:nonconvex_hyperparameters}.  SGD uses the same $\eta$ as the
reduction, and both methods use one fresh oracle call per round. We run the experiments for $T=20000$ steps and report the exact first $3000$ iterates.
 
\begin{table}[H]
\centering
\small
\renewcommand{\arraystretch}{1.13}
\begin{tabular}{lccccccc}
\hline
Function & $d$ & $M$ & $\beta$ & $\sigma$ & $S$ & $a$ & $(\eta,\gamma)$ \\
\hline
$f_1$
& $129$ & $1128.4$ & $8$ & $0.8$ & $65.127$ & $0.839792$
& $(1.64022\!\times\!10^{-3},1.86495\!\times\!10^{-4})$ \\
$f_2$
& $67$ & $2561.0111$ & $40$ & $3.2$ & $197.298$ & $0.707246$
& $(2.76268\!\times\!10^{-4},6.84549\!\times\!10^{-5})$ \\
$f_3$
& $64$ & $4224$ & $36$ & $3.2$ & $275.2$ & $0.861684$
& $(3.73995\!\times\!10^{-4},4.28676\!\times\!10^{-5})$ \\
$f_4$
& $64$ & $4097$ & $33$ & $0.4\sqrt{63}$ & $259.781$ & $0.818926$
& $(3.87749\!\times\!10^{-4},4.53614\!\times\!10^{-5})$ \\
\hline
\end{tabular}
\caption{objective functions and algorithm constants.}
\label{tab:nonconvex_hyperparameters}
\end{table}
 
\paragraph{Evaluation protocol.}
We use 128 held-out paired replications.  Within a replication, the two methods
start at the same point and use the same Rademacher noise.  The figures
show every iteration $t=1,\ldots,3000$ on a linear horizontal axis.  Their
right panels report the exact true squared gradient
$\|\nabla f_i(W_t)\|_2^2$ at each individual iteration, with no temporal
averaging or smoothing.  The summary table also includes the trajectory average $A_{i,3000}
:=\frac1{3000}\sum_{t=1}^{3000}\|\nabla f_i(W_t)\|_2^2$. Lines and table entries are means over our $128$ replications.
 
\begin{figure}[]
    \centering
    \includegraphics[width=\textwidth]
    {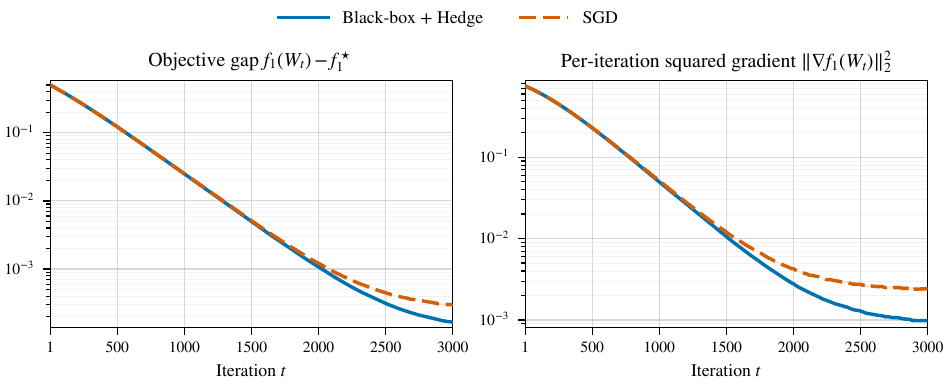}
    \caption{Objective gap and per-iteration squared gradient for $f_1$.}
    \label{fig:nonconvex_rotated_cosine}
\end{figure}
 
\begin{figure}[]
    \centering
    \includegraphics[width=\textwidth]
    {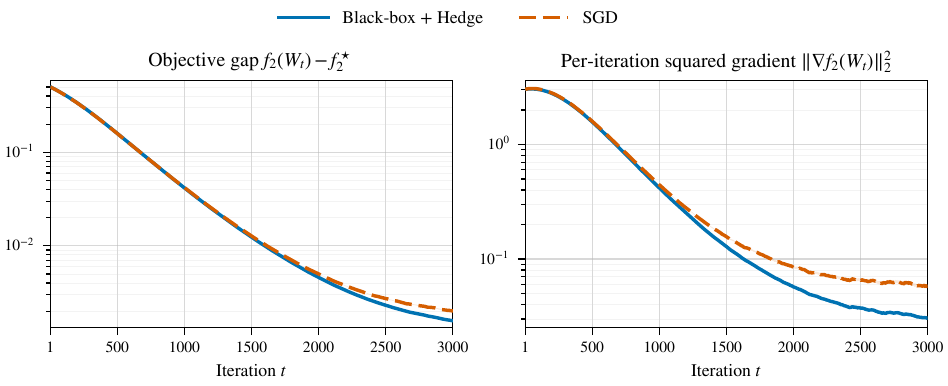}
    \caption{Objective gap and per-iteration squared gradient for $f_2$.}
    \label{fig:nonconvex_double_well}
\end{figure}
 
\begin{figure}[]
    \centering
    \includegraphics[width=\textwidth]
    {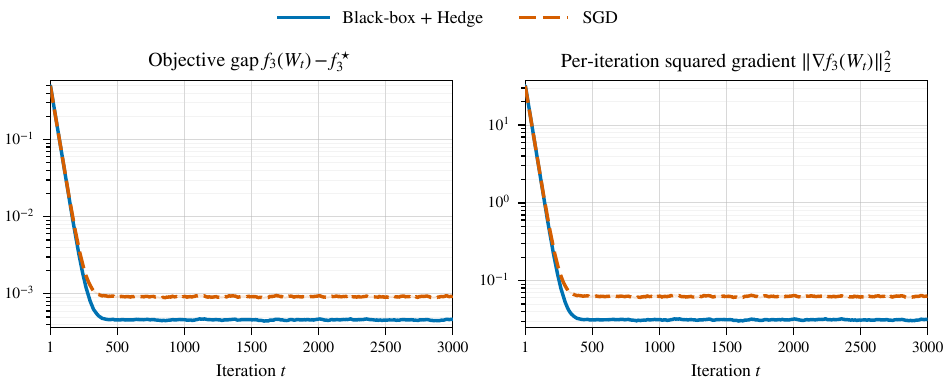}
    \caption{Objective gap and per-iteration squared gradient for $f_3$.}
    \label{fig:nonconvex_xy}
\end{figure}
 
\begin{figure}[]
    \centering
    \includegraphics[width=\textwidth]
    {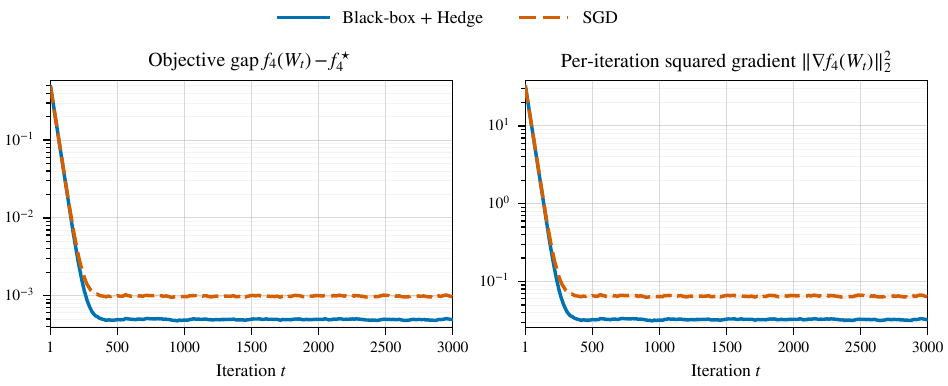}
    \caption{Objective gap and per-iteration squared gradient for $f_4$.}
    \label{fig:nonconvex_radial_cosine}
\end{figure}
 
\paragraph{Results at iteration $t=3000$.}
Table~\ref{tab:nonconvex_t3000} lists all three requested quantities: final
optimization error, trajectory-average squared gradient, and the instantaneous
squared gradient at $t=3000$.
 
\begin{table}[H]
\centering
\scriptsize
\setlength{\tabcolsep}{4.5pt}
\renewcommand{\arraystretch}{1.14}
\begin{tabular}{l l c c c}
\hline
Function & Method
& $f_i(W_{3000})-f_i^\star$
& $A_{i,3000}$
& $\|\nabla f_i(W_{3000})\|_2^2$ \\
\hline
$f_1$
& Black-box $+$ Hedge
& $1.6710\times10^{-4}$ & $0.102677$ & $9.6573\times10^{-4}$ \\
& SGD
& $2.9767\times10^{-4}$ & $0.103936$ & $2.3447\times10^{-3}$ \\
\hline
$f_2$
& Black-box $+$ Hedge
& $1.5907\times10^{-3}$ & $0.631159$ & $3.1510\times10^{-2}$ \\
& SGD
& $2.0250\times10^{-3}$ & $0.658224$ & $5.9215\times10^{-2}$ \\
\hline
$f_3$
& Black-box $+$ Hedge
& $4.6441\times10^{-4}$ & $0.485018$ & $3.1664\times10^{-2}$ \\
& SGD
& $9.1947\times10^{-4}$ & $0.509950$ & $6.2710\times10^{-2}$ \\
\hline
$f_4$
& Black-box $+$ Hedge
& $4.8169\times10^{-4}$ & $0.470959$ & $3.1791\times10^{-2}$ \\
& SGD
& $9.7010\times10^{-4}$ & $0.496628$ & $6.4027\times10^{-2}$ \\
\hline
\end{tabular}
\caption{Sample means after 3000 iterations over 128 held-out paired
replications.  The middle metric column is the trajectory average; the last
column and the right figure panels use the instantaneous true gradient.}
\label{tab:nonconvex_t3000}
\end{table}
 
Table~\ref{tab:nonconvex_paired_ci} reports paired, normal-approximation
\(95\%\) Monte Carlo confidence intervals for the mean difference between
Black-box \(+\) Hedge and SGD. For each objective and metric, the interval is
computed from the replication-level paired differences as
\(\bar{\Delta}\pm1.96\,s_{\Delta}/\sqrt{128}\). All twelve intervals lie
strictly below zero, providing evidence in favor of Black-box \(+\) Hedge
under the fixed experimental protocol. These are individual, unadjusted
intervals rather than simultaneous confidence statements across objectives
and metrics.
 
\begin{table}[H]
\centering
\scriptsize
\setlength{\tabcolsep}{4.0pt}
\renewcommand{\arraystretch}{1.14}
\begin{tabular}{l c c c}
\hline
Function
& $\Delta[f_i(W_{3000})-f_i^\star]$
& $\Delta A_{i,3000}$
& $\Delta\|\nabla f_i(W_{3000})\|_2^2$ \\
\hline
$f_1$
& $[-1.3435,-1.2679]\times10^{-4}$
& $[-1.2646,-1.2529]\times10^{-3}$
& $[-1.4238,-1.3343]\times10^{-3}$ \\
$f_2$
& $[-4.4739,-4.2114]\times10^{-4}$
& $[-2.7190,-2.6940]\times10^{-2}$
& $[-2.8754,-2.6654]\times10^{-2}$ \\
$f_3$
& $[-4.7121,-4.3891]\times10^{-4}$
& $[-2.5042,-2.4821]\times10^{-2}$
& $[-3.2155,-2.9938]\times10^{-2}$ \\
$f_4$
& $[-5.0825,-4.6858]\times10^{-4}$
& $[-2.5793,-2.5546]\times10^{-2}$
& $[-3.3544,-3.0926]\times10^{-2}$ \\
\hline
\end{tabular}
\caption{Paired 95\% confidence intervals for the differences between the two methods.}
\label{tab:nonconvex_paired_ci}
\end{table}
 
Finally, we summarize our numerical experiments as follows. Relative to the SGD baseline, our Black-box $+$ Hedge lowers the final objective gap (optimality gap) by
$43.9\%$, $21.4\%$, $49.5\%$, and $50.3\%$ on $f_1,f_2,f_3,$ and $f_4$,
respectively.  The corresponding reductions in the average squared gradient norms $A_{i,3000}$ are $1.2\%$,
$4.1\%$, $4.9\%$, and
$5.2\%$, while the instantaneous squared gradient reductions are $58.8\%$,
$46.8\%$, $49.5\%$, and $50.3\%$.  For each of the four objectives, all 128 paired replications favor
the reduction in all three reported metrics, and every corresponding
paired confidence interval lies strictly below zero.

\section{Discussion}
\label{sec:discussion}
In this paper, we show that ordinary, undiscounted static regret is sufficient
for a black-box online-to-nonconvex conversion once the online learner is
placed in the preconditioner slot. Under the assumptions used here, the
construction provides a black-box static-regret interface of the type sought by
\citet{chen2024open}, with the optimal rate. A predictable tracker supplies the descent
signal, the OCO oracle selects its geometry, and a fresh stochastic gradient
reveals a linear loss. The oracle is run once over the full horizon and is
compared only with the fixed identity preconditioner; no dynamic, shifting,
adaptive, or discounted regret guarantee is required.

We also establish the necessity of the variance term and the tightness
of our guarantee. For sufficiently large $T$, any conversion making a
constant number of stochastic-oracle queries per round, whose guarantee
is stated through static regret against a comparator that the online
learner can play, must include an additive term of order
$\sigma\sqrt{M\beta/T}$. For our conversion, the full guarantee, including
its dependence on regret, is tight up to universal constants for every
online learner. Comparator feasibility is essential to the impossibility
result: allowing a comparator outside the decision sets permits
variance-free guarantees even for conversions that ignore the online
learner.

The same interface covers both regularity regimes studied in this paper. For smooth objectives, an OCO oracle with \(O(\sqrt{T})\) static regret recovers the classical \(O(T^{-1/2})\) stationarity rate. For objectives without Lipschitz-continuous gradients, random segment queries and a normalized predictable tracker replace the smooth descent argument, while the same static-regret interface attains the optimal asymptotic \(O(T^{-2/7})\) convergence rate for the Goldstein-type relaxed stationarity measure.

Our numerical experiments complement these theoretical results by instantiating the smooth reduction with a two-expert Hedge oracle on four bounded, smooth nonconvex objectives. Under the paired protocol and the same stochastic-oracle budget, the reduction achieves lower objective gaps and squared-gradient measures than SGD across the tested settings. These experiments provide a controlled, end-to-end illustration of the reduction; evaluating the framework on broader objective classes and large-scale data-driven problems, such as large language model training, remains an important direction for future work.

\bibliographystyle{plainnat}
\bibliography{refs}
\appendix
\section{Mathematical Tools}\label{app:math_tools}
\begin{lemma}\label{lemma:convex}
    When $f$ is $\beta$-smooth, i.e. $\|\nabla f(x)-\nabla f(y)\|\le \beta\|x-y\|$, then we have
    $$f(y)\le f(x)+\inner{\nabla f(x)}{y-x}+\frac{\beta}{2}\|y-x\|^2.$$
\end{lemma}

\begin{lemma}\label{lemma:proj}
    Let $\Pi_C$ be the Euclidean projection onto a closed convex set $C\subset\RR^d$. For any $z\in\RR^d$ and any $y\in C$, we have 
    \[
    \|\Pi_C(z)-y\|\le \|z-y\|,\ \forall y,z.
    \]
\end{lemma}
Next, we introduce the Hedge algorithm in online convex optimization, which we used in designing the optimizer in Section \ref{sec:num_experiment}. Hedge is the special case of FTRL in which the
decision set is the probability simplex
\[
    \Delta_K
    :=
    \left\{
        w\in\RR^K
        :
        w_j\geq0
        \text{ for all } j\in[K],
        \
        \sum_{j=1}^K w_j=1
    \right\},
\]
the losses are linear, and the regularizer is the Kullback-Leibler divergence
to a prior \(\pi\in\Delta_K\) with \(\pi_j>0\) for every \(j\in[K]\),
\[
    \operatorname{KL}(w\|\pi)
    =
    \sum_{j=1}^K w_j\log\frac{w_j}{\pi_j}.
\]
At round \(t\) the learner commits to \(w_t\in\Delta_K\) before observing a
loss vector \(\ell_t\in\RR^K\), and then incurs the loss
\(\inner{\ell_t}{w_t}\).  Writing
\(C_t:=\sum_{s=1}^t\ell_s\) for the cumulative loss vector, the static regret
against a fixed comparator \(u\in\Delta_K\) is
\[
    \operatorname{Reg}_T(u)
    :=
    \sum_{t=1}^T\inner{\ell_t}{w_t}
    -
    \sum_{t=1}^T\inner{\ell_t}{u},
\]
and the special case \(u=e_j\), the \(j\)-th vertex, compares the learner with
expert \(j\) alone.

\begin{algorithm}[H]
\caption{Hedge (entropic follow-the-regularized-leader)}
\label{alg:hedge}
\small
\begin{algorithmic}[1]
\Require Number of experts \(K\), prior \(\pi\in\Delta_K\) with \(\pi_j>0\)
for every \(j\in[K]\), positive nonincreasing rates
\(\gamma_1\geq\gamma_2\geq\cdots>0\)
\State Initialize \(C_0=0\in\RR^K\).
\For{\(t=1,\ldots,T\)}
    \State Play
    \(
        w_t
        =
        \operatorname*{argmin}_{w\in\Delta_K}
        \left\{
            \inner{C_{t-1}}{w}
            +
            \frac{1}{\gamma_t}\operatorname{KL}(w\|\pi)
        \right\},
    \)
    i.e., for \(j\in[K]\), we set
    \[
        w_{t,j}
        =
        \frac{
            \pi_j\exp\!\left(-\gamma_tC_{t-1,j}\right)
        }{
            \sum_{k=1}^K\pi_k\exp\!\left(-\gamma_tC_{t-1,k}\right)
        }.
    \]
    \State Observe \(\ell_t\in\RR^K\) and incur \(\inner{\ell_t}{w_t}\).
    \State Update \(C_t=C_{t-1}+\ell_t\).
\EndFor
\end{algorithmic}
\end{algorithm}

\begin{lemma}
\label{lem:hedge_regret}
Run Algorithm~\ref{alg:hedge} with a constant rate \(\gamma_t\equiv\gamma>0\).
Then, for every \(u\in\Delta_K\),
\[
    \operatorname{Reg}_T(u)
    \leq
    \frac{\operatorname{KL}(u\|\pi)}{\gamma}
    +
    \frac{\gamma}{2}\sum_{t=1}^T\norm{\ell_t}_\infty^2 .
\]
In particular, with a uniform prior \(\pi=\bm1/K\) and
\(\gamma=\sqrt{2\log K/\sum_{t\leq T}\norm{\ell_t}_\infty^2}\),
\[
    \sum_{t=1}^T\inner{\ell_t}{w_t}
    -
    \min_{j\in[K]}\sum_{t=1}^T\ell_{t,j}
    \leq
    \sqrt{
        2\log K\sum_{t=1}^T\norm{\ell_t}_\infty^2
    }.
\]
\end{lemma}

Two features of Lemma~\ref{lem:hedge_regret} are the ones we rely on.  The
bound holds for \emph{arbitrary} bounded loss vectors, in particular for losses
generated adaptively by the learner's own past decisions, which is the
situation created by our reduction.  And because the comparator ranges over the
whole simplex, choosing \(u=e_j\) for the vertex that represents ``do nothing''
turns the guarantee into a safety statement: the learner is never much worse
than that designated baseline expert.

\section{Proofs in Section \ref{sec:thm}}\label{app:proof_thm}
In this section, we provide proofs in Section \ref{sec:thm}. First, we will provide and prove several lemmas that are useful throughout the paper. 
\begin{lemma}\label{lemma:gradient_bound}
Let $f$ be $\beta$-smooth with $f_*=\inf_{y\in\RR^d}f(y)>-\infty$. Then for every $x\in\RR^d$,
\[
\|\nabla f(x)\|\le\sqrt{2\beta(f(x)-f_*)}.
\]
In particular, $\|\nabla f(x_1)\|\le\sqrt{2M\beta}$ under Assumption~\ref{ass:nonconvex_function}, and $\|\nabla f(x)\|\le\sqrt{2M\beta}$ for every $x$ if $f(x)-f(y)\le M$ for all $x,y$.
\end{lemma}

\begin{proof}[Proof of Lemma \ref{lemma:gradient_bound}]
    Let $g=\nabla f(x)$ and set $y=x-g/\beta$. Since $f$ is $\beta$-smooth, we have
    \[
    f(y)\le f(x)+\inner{g}{y-x}+\frac{\beta}{2}\|y-x\|^2.
    \]
    Plugging $y-x=-g/\beta$ in, we have that 
    \[
    \inner{g}{y-x}=\frac{-1}{\beta}\|g\|^2,\ \frac{\beta}{2}\|y-x\|^2=\frac{1}{2\beta}\|g\|^2.
    \]
    Thus, we have $f(y)\le f(x)-\frac{1}{2\beta}\|g\|^2$. Since $f(y)\ge f_*$, we have 
    \[
    \|g\|^2\le 2\beta(f(x)-f_*).
    \]
    Taking the square root gives the first claim. The two special cases follow from $f(x_1)-f_*\le M$ and from $f(x)-f_*\le M$, respectively. We finish the proof.
\end{proof}

\begin{lemma}\label{lemma:young}
    For every $a\in(0,1]$ and all vectors $e,\Delta$, we have
    \[
    \|(1-a)e-\Delta\|^2\le (1-\frac{a}{2})\|e\|^2+\frac{4}{a}\|\Delta\|^2.
    \]
\end{lemma}
\begin{proof}[Proof of Lemma \ref{lemma:young}]
    For any $\lambda>0$, and any $u,v$, by algebra, we have
    \[
    \|u+v\|^2=\|u\|^2+\|v\|^2+2\inner{u}{v}\le (1+\lambda)\|u\|^2+(1+\frac{1}{\lambda})\|v\|^2.
    \]
    Set $u=(1-a)e$, $v=-\Delta$, $\lambda=a/2$. Then we have
    \[
    (1+\lambda)\|u\|^2=(1+\frac{a}{2})(1-a)^2\|e\|^2.
    \]
    For $a\in(0,1]$, again by algebra, we have $(1+\frac{a}{2})(1-a)^2=1-\frac{3}{2}a+\frac{1}{2}a^3\le 1-\frac{a}{2}$.
The last inequality holds because $\frac{-3}{2}a+\frac{1}{2}a^3\le \frac{-1}{2}a$. Secondly, we have $1+\frac{1}{\lambda}=1+\frac{2}{a}\le \frac{4}{a}$. Plugging this bound back, we have
\[
\|(1-a)e-\Delta\|^2\le (1-\frac{a}{2})\|e\|^2+\frac{4}{a}\|\Delta\|^2.
\]
We finish the proof.
\end{proof}
Equipped with these lemmas, we are ready to analyze the error of the predictable tracker $m_t$. From the update $x_{t+1}-x_t=-\eta U_tm_t$. The key is that $m_t$ should not be too far from the gradient $g_t=\nabla f(x_t)$. Specifically, we define the following quantities.
\[
e_t=m_t-g_t,\ E_{tr}=\sum_{t=1}^{T}\EE[\|e_t\|^2],\ S=\sum_{t=1}^{T}\EE[\|g_t\|^2]=\sum_{t=1}^{T}\EE[\|\nabla f(x_t)\|^2],\ V_T=\sum_{t=1}^{T}\EE[\|s_t-g_t\|^2]. 
\]
Regarding these quantities, we have the following important lemma.
\begin{lemma}\label{lemma:import1}
For our reduction algorithm \ref{alg:alg_known_T}, we have that
\[
E_{tr}\le \frac{6G^2}{a}+3aV_T+\frac{1}{128}S.
\]
\end{lemma}
\begin{proof}[Proof of Lemma \ref{lemma:import1}]
    Define $\xi_t=s_t-g_t$, then we have that $\xi_t$ is exogeneous zero-mean random noise. That is to say, if we denote $\cH_t$ as the $\sigma$ algebra generated by the iteration history up to round $t$ before we sample $s_t\sim \cO(x_t)$, i.e., $\cH_t=\sigma(\cbr{s_i}_{i<t}\cup\cbr{U_i}_{i\le t})$, which determines $x_t$, $m_t$, and $g_t$. Then we have $\EE[\xi_t|\cH_t]=0$, $\EE[\|\xi_t\|^2|\cH_t]\le \sigma^2$.
    
    We can rewrite $V_T$ as $\sum_{t=1}^{T}\EE[\|\xi_t\|^2]$ and denote $\Delta_t=g_{t+1}-g_t$.
    
    By induction on $t$, the vectors $x_t$, $m_t$, and $g_t$ have finite second moments, because $\|g_t\|\le\|\nabla f(x_1)\|+\beta\|x_t-x_1\|$ and $\EE[\|\xi_t\|^2|\cH_t]\le\sigma^2$. Hence all expectations below are finite.

    Let $z_t:=(1-a)m_t+as_t$, so that $m_{t+1}=z_t$. Using $s_t=g_t+\xi_t$ and $m_t=g_t+e_t$, we obtain
    \[
    z_t=(1-a)m_t+as_t=g_t+(1-a)e_t+a\xi_t.
    \]
    Recall that we define $\Delta_t=g_{t+1}-g_t$, then, we have
    \[
    z_t-g_{t+1}=(1-a)e_t+a\xi_t-\Delta_t.
    \]
    Since $m_{t+1}=z_t$, we have
    \[
    \|e_{t+1}\|^2=\|z_t-g_{t+1}\|^2=\|(1-a)e_t+a\xi_t-\Delta_t\|^2.
    \]
    Denoting $(1-a)e_t-\Delta_t$ as $r_t$ and expanding the square on the right side, we have 
    \[
    \EE[\|r_t+a\xi_t\|^2|\cH_t]=\|r_t\|^2+2a\inner{r_t}{\EE[\xi_t|\cH_t]}+a^2\EE[\|\xi_t\|^2|\cH_t].
    \]
    The equality holds because $r_t$ is measurable with respect to $\cH_t$.
    Because $\xi_t$ is conditionally zero-mean, we have $\EE[\|r_t+a\xi_t\|^2|\cH_t]=\|r_t\|^2+a^2\EE[\|\xi_t\|^2|\cH_t]$.
    For the term $r_t=(1-a)e_t-\Delta_t$, we apply Lemma \ref{lemma:young} to obtain
    \[
    \|r_t\|^2\le (1-\frac{a}{2})\|e_t\|^2+\frac{4}{a}\|\Delta_t\|^2.
    \]
    By smoothness, we have that \[
    \|\Delta_t\|=\|\nabla f(x_{t+1})-\nabla f(x_t)\|\le \beta\|x_{t+1}-x_t\|.
    \]
Therefore, we have that 
\[
\|x_{t+1}-x_t\|\le \eta\|m_t\|.
\]
Since $\eta=\frac{a}{64\beta}$, we obtain $\|\Delta_t\|\le \beta\eta\|m_t\|=\frac{a}{64}\|m_t\|$.
Thus, we have
\begin{align*}
    \frac{4}{a}\|\Delta_t\|^2\le \frac{4}{a}\frac{a^2}{64^2}\|m_t\|^2=\frac{a}{1024}\|g_t+e_t\|^2\le \frac{a}{512}\|g_t\|^2+\frac{a}{512}\|e_t\|^2.
\end{align*}
Combining all these together and taking expectation conditioned on the history $\cH_t$, we have
\[
\EE[\|e_{t+1}\|^2|\cH_t]\le (1-\frac{a}{2}+\frac{a}{512})\|e_t\|^2+\frac{a}{512}\|g_t\|^2+a^2\EE[\|\xi_t\|^2|\cH_t].
\]
On the right hand side, we can drop the conditional expectation because $e_t$ and $g_t$ are measurable with respect to $\cH_t$. 

Now, we take the total expectation and use $\frac{1}{2}-\frac{1}{512}>\frac{1}{3}$ to obtain
\[
\EE[\|e_{t+1}\|^2]\le (1-\frac{a}{3})\EE[\|e_t\|^2]+\frac{a}{512}\EE[\|g_t\|^2]+a^2\EE[\|\xi_t\|^2].
\]
Rearranging and telescoping, we have
\begin{align*}
\frac{a}{3}\sum_{t=1}^{T}\EE[\|e_t\|^2]\le \sum_{t=1}^{T}(\EE[\|e_{t}\|^2]-\EE[\|e_{t+1}\|^2])+\frac{a}{512}\sum_{t=1}^{T}\EE[\|g_t\|^2]+a^2\sum_{t=1}^{T}\EE[\|\xi_t\|^2].
\end{align*}
Since $m_1=0$, we have $e_1=-g_1$. By Lemma \ref{lemma:gradient_bound} and the initial gap $f(x_1)-\inf f\le M$, we have $\|e_1\|^2=\|\nabla f(x_1)\|^2\le 2M\beta=2G^2$. Hence
\[
\frac{a}{3}E_{tr}\le 2G^2+\frac{a}{512}S
+a^2V_T.\]
Thus, we have
\[
E_{tr}\le 6\frac{G^2}{a}+\frac{3}{512}S+3aV_T\le 6\frac{G^2}{a}+\frac{1}{128}S+3aV_T. 
\]
We finish the proof.
\end{proof}
Now, we provide the following lemma, which says that the total loss obtained through our Algorithm \ref{alg:alg_known_T} has a descent lower bound.
\begin{lemma}\label{lemma:descent_lower_bound}
    The iterates of the reduction satisfy
    \[
    G\sum_{t=1}^{T}\EE[L_t(U_t)]\ge -64\frac{G^2}{a}-\frac{a}{128}\sum_{t=1}^{T}\EE{\|m_t\|^2}.
    \]
\end{lemma}
\begin{proof}[Proof of Lemma \ref{lemma:descent_lower_bound}]
    By smoothness, we have that 
    \[
    f(x_{t+1})\le f(x_t)+\inner{g_t}{x_{t+1}-x_t}+\frac{\beta}{2}\|x_{t+1}-x_t\|^2.
    \]
    The update is $x_{t+1}-x_t=-\eta U_tm_t$, so
    \[
    f(x_{t+1})-f(x_t)\le -\eta\inner{g_t}{U_tm_t}+\frac{\beta\eta^2}{2}\|m_t\|^2.
    \]
    Now we relate the left hand side to the online convex optimization loss. Conditioned on $\cH_t$, the variables $U_t$ and $m_t$ are fixed, while $s_t$ is a fresh unbiased oracle sample. Thus,
    \[
    \EE[L_t(U_t)|\cH_t]=\EE[\frac{-1}{G}\inner{s_t}{U_tm_t}|\cH_t]=\frac{-1}{G}\inner{g_t}{U_tm_t}.
    \]
    Thus, taking conditional expectation on both sides of the smoothness inequality, we shall get
    \[
    \EE[f(x_{t+1})-f(x_t)|\cH_t]\le \eta G\EE[L_t(U_t)|\cH_t]+\frac{\beta\eta^2}{2}\EE[\|m_t\|^2|\cH_t].
    \]
    Taking the total expectation, we have
    \[
    \EE[f(x_{t+1})-f(x_t)]\le \eta G\EE[L_t(U_t)]+\frac{\beta\eta^2}{2}\EE[\|m_t\|^2].
    \]
    Sum over $t=1,2,\cdots,T$, we have 
    \[
    \EE[f(x_{T+1})-f(x_1)]\le \eta G\sum_{t=1}^{T}\EE[L_t(U_t)]+\frac{\beta\eta^2}{2}\sum_{t=1}^{T}\EE[\|m_t\|^2].
    \]
    By Assumption \ref{ass:nonconvex_function}, $f(x_1)-f(x_{T+1})\le M$, and hence
    \[
    -M\le \eta G\sum_{t=1}^{T}\EE[L_t(U_t)]+\frac{\beta\eta^2}{2}\sum_{t=1}^{T}\EE[\|m_t\|^2]
    \]
    Note that $\eta=\frac{a}{64\beta}$, we have $\frac{M}{\eta}=\frac{64M\beta}{a}=64\frac{G^2}{a}$, $\frac{\beta\eta}{2}=\frac{a}{128}$. Plugging these two equalities back and we finish the proof.
\end{proof}
\begin{lemma}\label{lemma:comparator_comparison}
    For the fixed comparator $U_*=I_d$, we have
    \[
    G\sum_{t=1}^{T}\EE[L_t(I_d)]\le -S+\sqrt{SE_{tr}}.    \]
\end{lemma}
\begin{proof}[Proof of Lemma \ref{lemma:comparator_comparison}]
    Recall that $L_t(I_d)=-\frac{1}{G}\inner{s_t}{m_t}$. Conditioned on $\cH_t$, we have $$G\EE[L_t(I_d)|\cH_t]=-\inner{g_t}{m_t}.$$
    Using $m_t=g_t+e_t$, we have
    \[
    -\inner{g_t}{m_t}=-\|g_t\|^2+\inner{g_t}{g_t-m_t}.
    \]
    Taking total expectations and summing over $t$, we have
    \[
    G\sum_{t=1}^{T}\EE[L_t(I_d)]=-S+\sum_{t=1}^{T}\EE[\inner{g_t}{g_t-m_t}].
    \]
    We apply the Cauchy-Schwarz inequality to get
    \[
    \sum_{t=1}^{T}\EE[\inner{g_t}{g_t-m_t}]\le (\sum_{t=1}^T\EE[\|g_t\|^2])^{1/2}(\sum_{t=1}^T\EE[\|g_t-m_t\|^2])^{1/2}=\sqrt{S}\sqrt{E_{tr}}.
    \]
    Plugging this back and we finish the proof.
\end{proof}
Equipped with these lemmas, we are finally ready to prove our main theoretical guarantee. 
\begin{proof}[Proof of Theorem \ref{thm:known_T}]
By the definition of the realized static regret,
\[
\sum_{t=1}^{T}L_t(U_t)
-
\sum_{t=1}^{T}L_t(I_d)
\le
\reg_T(\cA,I_d).
\]
Taking expectations, multiplying by \(G\), and recalling that
\[
\mathscr{R}_T(\cA,I_d)
=
\EE[\reg_T(\cA,I_d)],
\]
we obtain
\[
G\sum_{t=1}^{T}\EE[L_t(U_t)]
\le
G\sum_{t=1}^{T}\EE[L_t(I_d)]
+
G\mathscr{R}_T(\cA,I_d).
\]

By Lemma~\ref{lemma:descent_lower_bound},
\[
G\sum_{t=1}^{T}\EE[L_t(U_t)]
\ge
-64\frac{G^2}{a}
-
\frac{a}{128}
\sum_{t=1}^{T}\EE[\|m_t\|^2].
\]
Combining the preceding two inequalities gives
\[
-64\frac{G^2}{a}
-
\frac{a}{128}
\sum_{t=1}^{T}\EE[\|m_t\|^2]
\le
G\sum_{t=1}^{T}\EE[L_t(I_d)]
+
G\mathscr{R}_T(\cA,I_d).
\]
Lemma~\ref{lemma:comparator_comparison} yields
\[
G\sum_{t=1}^{T}\EE[L_t(I_d)]
\le
-S+\sqrt{SE_{tr}}.
\]
Consequently,
\[
S
\le
64\frac{G^2}{a}
+
G\mathscr{R}_T(\cA,I_d)
+
\sqrt{SE_{tr}}
+
\frac{a}{128}
\sum_{t=1}^{T}\EE[\|m_t\|^2].
\]

Since \(m_t=g_t+e_t\),
\[
\|m_t\|^2
\le
2\|g_t\|^2+2\|e_t\|^2,
\]
and hence
\[
\sum_{t=1}^{T}\EE[\|m_t\|^2]
\le
2S+2E_{tr}.
\]
Substituting this bound gives
\[
S
\le
64\frac{G^2}{a}
+
G\mathscr{R}_T(\cA,I_d)
+
\sqrt{SE_{tr}}
+
\frac{a}{64}S
+
\frac{a}{64}E_{tr}.
\]
Because \(a\le 1\) and
\[
\sqrt{xy}\le \frac{x}{4}+y
\qquad
\text{for all }x,y\ge 0,
\]
we obtain
\[
\frac{63}{64}S
\le
64\frac{G^2}{a}
+
G\mathscr{R}_T(\cA,I_d)
+
\frac{1}{4}S
+
\frac{65}{64}E_{tr}.
\]
Rearranging gives
\[
S
\le
\frac{4096}{47}\frac{G^2}{a}
+
\frac{64}{47}G\mathscr{R}_T(\cA,I_d)
+
\frac{65}{47}E_{tr}.
\]

Applying Lemma~\ref{lemma:import1}, we have
\begin{align*}
S
&\le
\frac{4096}{47}\frac{G^2}{a}
+
\frac{64}{47}G\mathscr{R}_T(\cA,I_d)
+
\frac{65}{47}
\left(
6\frac{G^2}{a}
+
3aV_T
+
\frac{1}{128}S
\right)\\
&=
\frac{4486}{47}\frac{G^2}{a}
+
\frac{64}{47}G\mathscr{R}_T(\cA,I_d)
+
\frac{195}{47}aV_T
+
\frac{65}{6016}S.
\end{align*}
Therefore,
\[
\frac{5951}{6016}S
\le
\frac{4486}{47}\frac{G^2}{a}
+
\frac{64}{47}G\mathscr{R}_T(\cA,I_d)
+
\frac{195}{47}aV_T,
\]
and hence
\[
S
\le
\frac{574208}{5951}\frac{G^2}{a}
+
\frac{8192}{5951}G\mathscr{R}_T(\cA,I_d)
+
\frac{24960}{5951}aV_T.
\]

Recall that
\[
a
=
\frac{G}{\max\{G,\sigma\sqrt{T}\}}.
\]
By Assumption~\ref{ass:gradient_oracle},
\[
V_T
=
\sum_{t=1}^{T}\EE[\|s_t-g_t\|^2]
\le
\sigma^2T.
\]
It follows that
\[
aV_T
\le
\frac{G\sigma^2T}
{\max\{G,\sigma\sqrt{T}\}}
\le
G\sigma\sqrt{T}.
\]
Moreover,
\[
\frac{G^2}{a}
=
G\max\{G,\sigma\sqrt{T}\}
\le
G^2+G\sigma\sqrt{T}.
\]
Substituting these bounds yields
\[
S
\le
\frac{574208}{5951}G^2
+
\frac{599168}{5951}G\sigma\sqrt{T}
+
\frac{8192}{5951}
G\mathscr{R}_T(\cA,I_d).
\]
Dividing by \(T\), using \(G=\sqrt{M\beta}\), and noting that
\[
\frac{574208}{5951}<97,
\qquad
\frac{599168}{5951}<101,
\]
we conclude that
\[
\frac{1}{T}
\sum_{t=1}^{T}
\EE[\|\nabla f(x_t)\|^2]
\le
\frac{101\sigma\sqrt{M\beta}}{\sqrt{T}}
+
\frac{8192}{5951}
\frac{\sqrt{M\beta}\mathscr{R}_T(\cA,I_d)}{T}
+
\frac{97M\beta}{T}.
\]
This proves the result.
\end{proof}
Finally, we prove Corollary~\ref{cor:ogd}.
\begin{proof}[Proof of Corollary \ref{cor:ogd}]
We work in $\cS^{d\times d}$ with the Frobenius inner product. For $U\in\cS^{d\times d}$, we have $L_t(U)=\inner{U}{\nabla_t}_F$ with
\[
\nabla_t=-\frac{1}{2G}\big(s_tm_t^\top+m_ts_t^\top\big),
\]
so $\cA$ is online gradient descent with stepsize $\gamma$, i.e., $U_{t+1}=\Pi_{\cU}(U_t-\gamma\nabla_t)$. Since $I_d\in\cU$, Lemma~\ref{lemma:proj}, applied in $\cS^{d\times d}$, gives
\[
\|U_{t+1}-I_d\|_F^2\le\|U_t-\gamma\nabla_t-I_d\|_F^2=\|U_t-I_d\|_F^2-2\gamma\inner{\nabla_t}{U_t-I_d}_F+\gamma^2\|\nabla_t\|_F^2.
\]
Since $L_t$ is linear, $L_t(U_t)-L_t(I_d)=\inner{\nabla_t}{U_t-I_d}_F$. Rearranging, summing over $t$, and using $\|U_1-I_d\|_F\le D$, we obtain
\[
\reg_T(\cA,I_d)\le\frac{D^2}{2\gamma}+\frac{\gamma}{2}\sum_{t=1}^{T}\|\nabla_t\|_F^2.
\]
Since $m_1=0$ and $m_{t+1}=(1-a)m_t+as_t$, each $m_t$ is a convex combination of $0,s_1,\ldots,s_{t-1}$, so $\|m_t\|\le S$. Hence $\|\nabla_t\|_F\le\|s_t\|\|m_t\|/G\le S^2/G$, and
\[
\reg_T(\cA,I_d)\le\frac{D^2}{2\gamma}+\frac{\gamma TS^4}{2G^2}=\frac{DS^2\sqrt T}{G},
\]
where the equality uses the choice of $\gamma$. Taking expectations, plugging the bound into Theorem~\ref{thm:known_T}, and using $G=\sqrt{M\beta}$, we finish the proof.
\end{proof}
\section{Proofs in Section \ref{sec:extension}}\label{app:proof_extension}
Throughout this subsection, we write
\[
    b:=1-a,
    \ 
    g_t:=\nabla f(y_t),
    \ 
    \xi_t:=s_t-g_t,
\]
and use the following notations:
\[
    c_t=1-b^t,
    \ 
    q_{t,s}=\frac{ab^{t-s}}{c_t},
    \  1\le s\le t.
\]
Let $\cH_t^-$ denote the complete history immediately after selecting $U_t$ and before drawing $r_t$, including all past segment variables, oracle samples, and OCO randomization. Define $\cH_t=\cH_t^-\vee\sigma(r_t)$, where $r_t$ is independent of $\cH_t^-$. Then $x_t,m_t,z_t,U_t,\delta_t,x_{t+1}$ are $\cH_t^-$-measurable, and $y_t,g_t$ are $\cH_t$-measurable.

By the conditional stochastic-oracle Assumption \ref{ass:gradient_oracle}, we have
\[
    \EE[s_t\mid \cH_t]=g_t,
    \ 
    \EE[\|s_t-g_t\|_2^2\mid \cH_t]\le\sigma^2,
\]
as well as the preconditioner bound
\[
    \|U\|_{\mathrm{op}}\le1
    \ \text{for every }U\in\cU_t.
\]
All expectations below are taken over the randomness of the OCO oracle, the
segment variables \(r_t\), the stochastic-gradient oracle, and the output
index \(\tau\).

Recall also the local stationarity measure
\begin{equation}
\label{eq:detailed_ns_stationarity_definition}
    \|\nabla f(x)\|^{[\lambda]}
    :=
    \inf_{\substack{p\in\mathcal P(\RR^d)\\
                    \EE_{X\sim p}[X]=x}}
    \left\{
        \left\|\EE_{X\sim p}[\nabla f(X)]\right\|_2
        +
        \lambda\EE_{X\sim p}\|X-x\|_2^2
    \right\}.
\end{equation}

The proof is organized through a sequence of auxiliary lemmas.  We first
replace the unavailable smooth descent lemma with an exact random-segment
identity.  Static regret against the identity matrix then controls the
predictable tracker.  We subsequently relate that stochastic tracker to a
weighted average of true gradients, bound the spatial dispersion of the
associated query points, and combine the two estimates into a local
stationarity certificate.

The first lemma explains the role of the random segment query.  Averaging the
directional derivative over the segment from \(x_t\) to \(x_{t+1}\) recovers
the exact function difference, even though no Lipschitz-continuous gradient is
available.
\begin{lemma}[Exact progress from a random segment query]
\label{lem:nonsmooth_knownT_segment_progress}
For every round \(t\), Algorithm~\ref{alg:nonsmooth_static_reduction}
satisfies
\[
    \eta\EE [L_t(U_t)]
    =
    \EE[f(x_{t+1})-f(x_t)].
\]
Consequently,
\[
    \sum_{t=1}^T\EE [L_t(U_t)]
    \ge-\frac M\eta.
\]
\end{lemma}

\begin{proof}
Let \(\mathcal H_t^-\) denote all information available immediately after
the OCO oracle has selected \(U_t\) and before \(r_t\) is drawn.  Conditional
on \(\mathcal H_t^-\), the vectors \(x_t,m_t,z_t,U_t,\delta_t\) are fixed and
\[
    \delta_t=-\eta U_tz_t,
    \ 
    x_{t+1}=x_t+\delta_t.
\]
Since \(r_t\sim\operatorname{Unif}[0,1]\), the definition
\(y_t=x_t+r_t\delta_t\) and the Newton-Leibniz condition imply
\begin{align}
    \EE_{r_t}\!\left[
        \langle\nabla f(y_t),\delta_t\rangle
        \mid\mathcal H_t^-
    \right]
    &=
    \int_0^1
        \langle\nabla f(x_t+r\delta_t),\delta_t\rangle\,dr
    =f(x_t+\delta_t)-f(x_t)
    =f(x_{t+1})-f(x_t).
\label{eq:detailed_ns_segment_integral}
\end{align}
The first equality is the definition of expectation under a uniform random
variable, the second equality is Newton-Leibniz along the segment from
\(x_t\) to \(x_t+\delta_t\).

After \(r_t\) has been drawn, we know that \(\mathcal H_t\) contains
\(\mathcal H_t^-\), \(r_t\), and hence \(y_t\) and \(g_t\).  Using the conditional
unbiasedness of the stochastic gradient oracle, we have
\begin{align}
    \EE[L_t(U_t)\mid\mathcal H_t]
    &=-\EE[\langle s_t,U_tz_t\rangle\mid\mathcal H_t]
    =-\langle g_t,U_tz_t\rangle.
\label{eq:detailed_ns_loss_unbiased}
\end{align}
Multiplying by \(\eta\) and using
\(\delta_t=-\eta U_tz_t\), we obtain
\begin{equation}
\label{eq:detailed_ns_loss_directional_derivative}
    \eta\EE[L_t(U_t)\mid\mathcal H_t]
    =\langle g_t,\delta_t\rangle.
\end{equation}
Hence, we obtain that
\begin{align}\label{eq:detailed_ns_exact_progress}
    \eta\EE[L_t(U_t)]=\EE[\inner{s_t}{\delta_t}]=\EE[\EE[\inner{s_t}{\delta_t}|\cH_t]]=\EE[\inner{g_t}{\delta_t}]=\EE[\EE[\inner{g_t}{\delta_t}|\cH_t^-]]=\EE[f(x_{t+1})-f(x_t)].
\end{align}
Summing \eqref{eq:detailed_ns_exact_progress} from \(t=1\) to \(T\) gives
\begin{align}
    \eta\sum_{t=1}^T\EE [L_t(U_t)]
    &=
    \sum_{t=1}^T\EE[f(x_{t+1})-f(x_t)]
    =\EE[f(x_{T+1})-f(x_1)].
\label{eq:detailed_ns_progress_telescope}
\end{align}
The second equality follows because every intermediate term
\(f(x_2),\ldots,f(x_T)\) cancels. By Assumption \ref{ass:nonsmooth}, we have
\[
    f(x_1)-f(x_{T+1})\le M.
\]
This inequality holds for every realized trajectory, so we have
\begin{equation}
\label{eq:detailed_ns_played_loss_lower}
    \sum_{t=1}^T\EE [L_t(U_t)]
    \ge -\frac{M}{\eta}.
\end{equation}
This proves both claims.
\end{proof}

The preceding lemma controls the cumulative loss of the matrices actually
played by the OCO oracle.  The next lemma transfers this control to the fixed
identity comparator.  It is the unique point at which the black-box static
regret guarantee enters the nonsmooth analysis.

\begin{lemma}[Identity-comparator correlation bound]
\label{lem:nonsmooth_knownT_identity_correlation}
The undiscounted static-regret guarantee against the fixed comparator \(I_d\)
implies
\[
    \sum_{t=1}^T\EE\langle s_t,z_t\rangle
    \le
    \frac M\eta+\mathscr R_T(\cA,I_d).
\]
\end{lemma}

\begin{proof}
The expected static-regret guarantee against the fixed comparator \(I_d\) is
\begin{equation}
\label{eq:detailed_ns_regret_definition}
    \sum_{t=1}^T\EE
    \bigl[L_t(U_t)-L_t(I_d)\bigr]
    \le \mathscr R_T(\cA,I_d).
\end{equation}
Rearranging \eqref{eq:detailed_ns_regret_definition} gives
\begin{equation}
\label{eq:detailed_ns_regret_rearranged}
    \sum_{t=1}^T\EE [L_t(U_t)]
    \le
    \sum_{t=1}^T\EE [L_t(I_d)]
    +\mathscr R_T(\cA,I_d).
\end{equation}
Because
\[
    L_t(I_d)
    =-\langle s_t,I_dz_t\rangle
    =-\langle s_t,z_t\rangle,
\]
combining
\eqref{eq:detailed_ns_played_loss_lower} and
\eqref{eq:detailed_ns_regret_rearranged} yields
\[
    -\frac{M}{\eta}
    \le
    -\sum_{t=1}^T\EE\langle s_t,z_t\rangle
    +\mathscr R_T(\cA,I_d).
\]
Moving the correlation term to the left and \(M/\eta\) to the right gives
\begin{equation}
\label{eq:detailed_ns_correlation_bound}
    \sum_{t=1}^T\EE\langle s_t,z_t\rangle
    \le
    \frac{M}{\eta}+\mathscr R_T(\cA,I_d).
\end{equation}
\end{proof}
The identity-comparator loss is a correlation between the stochastic
gradient and the normalized predictable direction.  The next lemma converts
this cumulative correlation into a bound on the average magnitude of the
tracker.  Its proof introduces the smoothed norm
\(\Phi_\rho(m)=\sqrt{\|m\|_2^2+\rho^2}\) and shows explicitly that the
contraction by \(b=1-a\) creates a potential leakage proportional to
\(\|m_t\|_2\).
\begin{lemma}[Static regret controls the predictable tracker]
\label{lem:nonsmooth_knownT_tracker_norm}
The projected tracker in Algorithm~\ref{alg:nonsmooth_static_reduction}
satisfies
\[
    \frac1T\sum_{t=2}^T\EE[\|m_t\|_2]
    \le
    \frac{2M}{\eta T}
    +\frac{2\mathscr R_T(\cA,I_d)}T
    +\rho
    +\frac{a(L^2+\sigma^2)}\rho.
\]
\end{lemma}

\begin{proof}
Define the smoothed Euclidean norm
\[
    \Phi_\rho(u):=\sqrt{\|u\|_2^2+\rho^2}.
\]
Its gradient and Hessian are
\begin{align}
    \nabla\Phi_\rho(u)
    &=\frac{u}{\sqrt{\|u\|_2^2+\rho^2}},\\
    \nabla^2\Phi_\rho(u)
    &=
    \frac{I_d}{\Phi_\rho(u)}
    -
    \frac{uu^\top}{\Phi_\rho(u)^3}.
\label{eq:detailed_ns_phi_derivatives}
\end{align}
For every \(v\in\RR^d\), Cauchy--Schwarz implies
\((u^\top v)^2\le\|u\|_2^2\|v\|_2^2\), and therefore
\begin{align*}
    v^\top\nabla^2\Phi_\rho(u)v
    &=
    \frac{\|v\|_2^2}{\Phi_\rho(u)}
    -
    \frac{(u^\top v)^2}{\Phi_\rho(u)^3}
    \ge0,
\end{align*}
while
\[
    v^\top\nabla^2\Phi_\rho(u)v
    \le
    \frac{\|v\|_2^2}{\Phi_\rho(u)}
    \le
    \frac{\|v\|_2^2}{\rho}.
\]
Hence
\[
    0\preceq\nabla^2\Phi_\rho(u)\preceq\frac1\rho I_d,
\]
so \(\Phi_\rho\) is \(1/\rho\)-smooth.  The standard smoothness
inequality, obtained by integrating the gradient along the segment from
\(u\) to \(u+v\), is therefore
\begin{equation}
\label{eq:detailed_ns_phi_smoothness}
    \Phi_\rho(u+v)
    \le
    \Phi_\rho(u)
    +\langle\nabla\Phi_\rho(u),v\rangle
    +\frac1{2\rho}\|v\|_2^2.
\end{equation}

Apply \eqref{eq:detailed_ns_phi_smoothness} with
\(u=bm_t\) and \(v=as_t\).  Since
\[
    z_t
    =\frac{bm_t}{\sqrt{b^2\|m_t\|_2^2+\rho^2}}
    =\nabla\Phi_\rho(bm_t),
\]
we obtain
\begin{equation}
\label{eq:detailed_ns_phi_before_projection}
    \Phi_\rho(bm_t+as_t)
    \le
    \Phi_\rho(bm_t)
    +a\langle s_t,z_t\rangle
    +\frac{a^2}{2\rho}\|s_t\|_2^2.
\end{equation}
The tracker update is
\[
    m_{t+1}=\Pi_{\BB_L}(bm_t+as_t),
    \ 
    \BB_L:=\{u\in\RR^d:\|u\|_2\le L\}.
\]
Projection onto a ball centered at the origin cannot increase the Euclidean
norm.  Since \(\Phi_\rho(u)\) is an increasing function of \(\|u\|_2\),
\[
    \Phi_\rho(m_{t+1})
    \le
    \Phi_\rho(bm_t+as_t).
\]
Combining this inequality with
\eqref{eq:detailed_ns_phi_before_projection} and rearranging gives the
one-step lower bound
\begin{equation}
\label{eq:detailed_ns_phi_one_step}
    a\langle s_t,z_t\rangle
    \ge
    \Phi_\rho(m_{t+1})-\Phi_\rho(bm_t)
    -\frac{a^2}{2\rho}\|s_t\|_2^2.
\end{equation}

We next quantify the loss of potential caused by replacing \(m_t\) with
\(bm_t\).  Fix an arbitrary vector \(m\) and write \(r:=\|m\|_2\).  By
rationalizing the difference,
\begin{align}
    \Phi_\rho(m)-\Phi_\rho(bm)
    &=
    \sqrt{r^2+\rho^2}-\sqrt{b^2r^2+\rho^2}
    =
    \frac{(1-b^2)r^2}
    {\sqrt{r^2+\rho^2}+\sqrt{b^2r^2+\rho^2}}.
\label{eq:detailed_ns_leakage_exact}
\end{align}
Because \(b=1-a\), we have
\[
    1-b^2
    =1-(1-a)^2
    =a(2-a)
    \ge a.
\]
Furthermore,
\[
    \sqrt{r^2+\rho^2}\le r+\rho
\]
and
\[
    \sqrt{b^2r^2+\rho^2}
    \le br+\rho
    \le r+\rho,
\]
where the last inequality uses \(0\le b\le1\).

Consequently, the
denominator in \eqref{eq:detailed_ns_leakage_exact} is at most
\(2(r+\rho)\), and hence
\begin{equation}
\label{eq:detailed_ns_leakage_intermediate}
    \Phi_\rho(m)-\Phi_\rho(bm)
    \ge
    \frac{ar^2}{2(r+\rho)}.
\end{equation}
Finally, by algebra, we have $\frac{r^2}{r+\rho}-(r-\rho)
    =\frac{\rho^2}{r+\rho}\ge0,$
so \(r^2/(r+\rho)\ge r-\rho\).  Substituting this into
\eqref{eq:detailed_ns_leakage_intermediate} yields
\begin{equation}
\label{eq:detailed_ns_leakage_lower}
    \Phi_\rho(m)-\Phi_\rho(bm)
    \ge
    \frac a2(\|m\|_2-\rho).
\end{equation}

Sum \eqref{eq:detailed_ns_phi_one_step} over \(t=1,\ldots,T\).  The
potential terms can be regrouped exactly as
\begin{align}
    &\sum_{t=1}^T
    \bigl(\Phi_\rho(m_{t+1})-\Phi_\rho(bm_t)\bigr)
    =
    \Phi_\rho(m_{T+1})-\Phi_\rho(bm_1)
    +
    \sum_{t=2}^T
    \bigl(\Phi_\rho(m_t)-\Phi_\rho(bm_t)\bigr).
\label{eq:detailed_ns_potential_regrouping}
\end{align}
To verify this identity, observe that the positive terms on the left are $\cbr{\Phi_\rho(m_2),\ldots,\Phi_\rho(m_{T+1})}$,
while the negative terms are $\cbr{\Phi_\rho(bm_1),\ldots,\Phi_\rho(bm_T)}$.
Pairing the terms with indices \(2,\ldots,T\) leaves precisely the two
boundary terms displayed above.

Since \(m_1=0\), we have $\Phi_\rho(bm_1)=\Phi_\rho(0)=\rho.$
Also, \(\Phi_\rho(m_{T+1})\ge\rho\), so the boundary difference in
\eqref{eq:detailed_ns_potential_regrouping} is nonnegative.  Applying
\eqref{eq:detailed_ns_leakage_lower} to every remaining term gives
\begin{equation}
\label{eq:detailed_ns_potential_lower}
    \sum_{t=1}^T
    \bigl(\Phi_\rho(m_{t+1})-\Phi_\rho(bm_t)\bigr)
    \ge
    \frac a2\sum_{t=2}^T(\|m_t\|_2-\rho).
\end{equation}
Combining the sum of \eqref{eq:detailed_ns_phi_one_step} with
\eqref{eq:detailed_ns_potential_lower} yields
\[
    a\sum_{t=1}^T\langle s_t,z_t\rangle
    \ge
    \frac a2\sum_{t=2}^T(\|m_t\|_2-\rho)
    -\frac{a^2}{2\rho}\sum_{t=1}^T\|s_t\|_2^2.
\]
Multiplying by \(2/a\) and moving terms gives the pathwise inequality
\begin{equation}
\label{eq:detailed_ns_tracker_pathwise}
    \sum_{t=2}^T\|m_t\|_2
    \le
    2\sum_{t=1}^T\langle s_t,z_t\rangle
    +\rho(T-1)
    +\frac a\rho\sum_{t=1}^T\|s_t\|_2^2.
\end{equation}

We now bound the second moment of \(s_t\).  Conditional on \(y_t\) and the
past, \(s_t=g_t+\xi_t\), and hence
\begin{align}
    \EE[\|s_t\|_2^2\mid \cH_t]
    &=
    \|g_t\|_2^2
    +2\left\langle
        g_t,
        \EE[\xi_t\mid \cH_t]
    \right\rangle
    +
    \EE[\|\xi_t\|_2^2\mid \cH_t]
    \le L^2+\sigma^2.
\label{eq:detailed_ns_sample_second_moment}
\end{align}
The cross term vanishes because the noise is conditionally mean zero, and
the last inequality uses \(\|g_t\|_2\le L\) and the conditional variance
bound.  Taking total expectation preserves this bound.

Taking expectations in \eqref{eq:detailed_ns_tracker_pathwise}, applying
\eqref{eq:detailed_ns_correlation_bound} and
\eqref{eq:detailed_ns_sample_second_moment}, and using \((T-1)/T\le1\), we
obtain
\begin{align}
    \frac1T\sum_{t=2}^T\EE[\|m_t\|_2]
    &\le
    \frac2T
    \left(
        \frac M\eta+\mathscr R_T(\cA,I_d)
    \right)
    +\frac{\rho(T-1)}T
    +\frac{a(L^2+\sigma^2)}\rho
    \notag\\
    &\le
    \frac{2M}{\eta T}
    +\frac{2\mathscr R_T(\cA,I_d)}T
    +\rho
    +\frac{a(L^2+\sigma^2)}\rho.
\label{eq:detailed_ns_tracker_expected}
\end{align}
This is the claimed tracker bound.
\end{proof}
The tracker \(m_t\) is formed from stochastic samples, whereas the
stationarity measure is defined using true gradients.  We therefore introduce
a proof-only tracker \(h_t\) driven by \(g_t=\nabla f(y_t)\).  The next lemma
uses projection nonexpansiveness and conditional unbiasedness to control
\(m_t-h_t\), without invoking any continuity property of the gradient.

\begin{lemma}
\label{lem:nonsmooth_knownT_tracker_noise}
Define the true-gradient tracker by
\[
    h_1=0,
    \ 
    h_{t+1}=bh_t+a\nabla f(y_t).
\]
Then, for every \(t\ge1\),
\[
    \EE[\|m_t-h_t\|_2^2]
    \le
    \frac{a\sigma^2}{2-a},
\]
and consequently
\[
    \EE[\|m_t-h_t\|_2]
    \le
    \sigma\sqrt{\frac a{2-a}}.
\]
\end{lemma}

\begin{proof}
Define
\begin{equation}
\label{eq:detailed_ns_true_tracker}
    h_1:=0,
    \ 
    h_{t+1}:=bh_t+ag_t.
\end{equation}
Unrolling this recursion gives
\begin{equation}
\label{eq:detailed_ns_true_tracker_unrolled}
    h_{t+1}
    =a\sum_{s=1}^tb^{t-s}g_s.
\end{equation}
Indeed, the claim holds for \(t=1\), and substituting the expression for
\(h_t\) into \(h_{t+1}=bh_t+ag_t\) proves the next case.  Since
\(\|g_s\|_2\le L\),
\begin{align}
    \|h_{t+1}\|_2
    &\le
    a\sum_{s=1}^tb^{t-s}\|g_s\|_2
    \le
    aL\sum_{j=0}^{t-1}b^j
    =L(1-b^t)
    =c_tL
    \le L.
\label{eq:detailed_ns_true_tracker_bounded}
\end{align}
Thus \(h_{t+1}\in\BB_L\).

Let \(e_t:=m_t-h_t\).  Because Euclidean projection is nonexpansive and
\(h_{t+1}=\Pi_{\BB_L}(h_{t+1})\),
\begin{align}
    \|e_{t+1}\|_2
    &=
    \left\|
        \Pi_{\BB_L}(bm_t+as_t)
        -\Pi_{\BB_L}(h_{t+1})
    \right\|_2
    \le
    \|bm_t+as_t-h_{t+1}\|_2
    \notag\\
    &=
    \|b(m_t-h_t)+a(s_t-g_t)\|_2
    =\|be_t+a\xi_t\|_2.
\label{eq:detailed_ns_noise_projection}
\end{align}
Since \(e_t\) depends only on observations from earlier rounds, it is \(\mathcal H_t^{-}\)-measurable and hence also \(\mathcal H_t\)-measurable. Thus, we have
\(\EE[\xi_t\mid \cH_t]=0\).  Squaring
\eqref{eq:detailed_ns_noise_projection} and taking conditional expectation, we have 
\begin{align}
    \EE[\|e_{t+1}\|_2^2\mid \cH_t]
    &\le
    b^2\|e_t\|_2^2
    +2ab\left\langle
        e_t,
        \EE[\xi_t\mid \cH_t]
    \right\rangle
    +a^2\EE[\|\xi_t\|_2^2\mid \cH_t]
    \le b^2\|e_t\|_2^2+a^2\sigma^2.
\label{eq:detailed_ns_noise_recursion}
\end{align}
Since \(e_1=0\), iterating
\eqref{eq:detailed_ns_noise_recursion} and taking total expectation gives
\begin{align}
    \EE[\|e_t\|_2^2]
    &\le
    a^2\sigma^2\sum_{j=0}^{t-2}b^{2j}
    \le
    \frac{a^2\sigma^2}{1-b^2}
    =
    \frac{a^2\sigma^2}{a(2-a)}
    =
    \frac{a\sigma^2}{2-a}.
\label{eq:detailed_ns_noise_squared}
\end{align}
The second inequality uses the infinite geometric series, and the equality
\(1-b^2=a(2-a)\) follows from \(b=1-a\).  Jensen's inequality for the
concave square-root function now yields
\begin{equation}
\label{eq:detailed_ns_noise_first_moment}
    \EE[\|m_t-h_t\|_2]
    =\EE[\|e_t\|_2]
    \le
    \sqrt{\EE[\|e_t\|_2^2}]
    \le
    \sigma\sqrt{\frac a{2-a}}.
\end{equation}
This proves both assertions.
\end{proof}

The geometric output weights are chosen to coincide exactly with the
coefficients obtained by unrolling the true-gradient tracker.  The following
lemma combines that algebraic identity with the distribution of \(\tau\).
This converts the mean gradient of the output witness into an unweighted
average of tracker norms, plus the stochastic tracking error and one terminal
normalization term.

Since $h_{t+1}=a\sum_{s=1}^{t}b^{t-s}g_s$, we define the re-normalized coefficients as
\[
c_t:=a\sum_{s=1}^tb^{t-s}=1-b^t, q_{t,s}:=\frac{ab^{t-s}}{c_t},
\]
then $\sum_{s=1}^{t}q_{t,s}=1$ and $\cbr{q_{t,s}}_{s=1}^t$ form a discrete distribution.

By some algebra, if we sample according to $\PP(Y=y_s)=q_{t,s}$, then we have $\EE[Y]=\sum_{s=1}^tq_{t,s}y_s=\bar y_t$ and $\EE[\nabla f(Y)]=\sum_{s=1}^tq_{t,s}g_s=\frac{h_{t+1}}{c_t}$.
\begin{lemma}[Weighted output-gradient bound]
\label{lem:nonsmooth_knownT_output_gradient}
Let
\(g_t=\nabla f(y_t)\).  The output of
Algorithm~\ref{alg:nonsmooth_static_reduction} satisfies
\[
    \EE\left[
        \left\|
            \sum_{s=1}^{\tau}q_{\tau,s}g_s
        \right\|_2
    \right]
    \le
    \frac1T\sum_{t=2}^{T}\EE[\|m_t\|_2]
    +\sigma\sqrt{\frac a{2-a}}
    +\frac L{aT}.
\]
\end{lemma}

\begin{proof}
Let \(h_t\) be the true-gradient tracker defined in
Lemma~\ref{lem:nonsmooth_knownT_tracker_noise}. First, recall that
\begin{align}
    \sum_{s=1}^tq_{t,s}
    &=
    \frac a{c_t}\sum_{s=1}^tb^{t-s}
    =
    \frac a{c_t}\sum_{j=0}^{t-1}b^j
    =
    \frac a{c_t}\frac{1-b^t}{1-b}
    =1,
\label{eq:detailed_ns_q_sum_one}
\end{align}
where \(1-b=a\) and \(c_t=1-b^t\). Thus, by \eqref{eq:detailed_ns_true_tracker_unrolled} and the definition of
\(q_{t,s}\), we have
\begin{equation}
\label{eq:detailed_ns_weighted_gradient}
    \sum_{s=1}^tq_{t,s}g_s
    =
    \frac1{c_t}
    a\sum_{s=1}^tb^{t-s}g_s
    =
    \frac{h_{t+1}}{c_t}.
\end{equation}

The output index is sampled according to
\[
    \Pr(\tau=t)
    =
    \begin{cases}
        c_t/T, & 1\le t<T,\\[1mm]
        c_T/(aT), & t=T.
    \end{cases}
\]
For completeness, these probabilities sum to one because
\begin{align*}
    \sum_{t=1}^{T-1}\frac{c_t}{T}+\frac{c_T}{aT}
    &=
    \frac1T
    \left(
        T-1-\sum_{t=1}^{T-1}b^t+\frac{1-b^T}{a}
    \right)=
    \frac1T
    \left(
        T-1-\frac{b(1-b^{T-1})}{a}
        +\frac{1-b^T}{a}
    \right)\\
    &=
    \frac1T\left(T-1+\frac{1-b}{a}\right)
    =1.
\end{align*}
Using \eqref{eq:detailed_ns_weighted_gradient} and conditioning on the
realized trajectory before drawing \(\tau\),
\begin{align}
    \EE\left[
        \left\|
            \sum_{s=1}^{\tau}q_{\tau,s}g_s
        \right\|_2
    \right]
    &=
    \sum_{t=1}^T
    \Pr(\tau=t)
    \EE\left[\frac{\|h_{t+1}\|_2}{c_t}\right]
    =
    \frac1T\sum_{t=1}^{T-1}\EE[\|h_{t+1}\|_2]
    +
    \frac1{aT}\EE[\|h_{T+1}\|_2].
\label{eq:detailed_ns_output_gradient_expand}
\end{align}
For every \(t<T\), the triangle inequality and
\eqref{eq:detailed_ns_noise_first_moment} imply
\[
    \EE[\|h_{t+1}\|_2]
    \le
    \EE[\|m_{t+1}\|_2]
    +\EE[\|m_{t+1}-h_{t+1}\|_2]
    \le
    \EE[\|m_{t+1}\|_2]
    +\sigma\sqrt{\frac a{2-a}}.
\]
Also, \eqref{eq:detailed_ns_true_tracker_bounded} gives
\[
    \|h_{T+1}\|_2\le c_TL\le L.
\]
Substituting these two bounds into
\eqref{eq:detailed_ns_output_gradient_expand}, changing the index
\(t+1\) to \(t\) in the first sum, and using \((T-1)/T\le1\), gives
\begin{align}
    \EE\left[
        \left\|
            \sum_{s=1}^{\tau}q_{\tau,s}g_s
        \right\|_2
    \right]
    &\le
    \frac1T\sum_{t=2}^{T}\EE[\|m_t\|_2]
    +\frac{T-1}{T}\sigma\sqrt{\frac a{2-a}}
    +\frac{L}{aT}
    \le
    \frac1T\sum_{t=2}^{T}\EE[\|m_t\|_2]
    +\sigma\sqrt{\frac a{2-a}}
    +\frac{L}{aT}.
\label{eq:detailed_ns_output_gradient_bound}
\end{align}
This is the desired output-gradient bound.
\end{proof}

The averaged-gradient estimate alone does not yet yield local stationarity:
the witness distribution must also be concentrated around its mean.  The next
lemma controls this spatial variance.  Its proof separates the displacement
inside each randomly queried segment from the accumulated motion of the base
iterates, and then computes the coefficient of every individual increment
\(\delta_k\).

\begin{lemma}[Localization of the single-run output]
\label{lem:nonsmooth_knownT_output_localization}
The weighted output points satisfy
\[
    \EE\left[
        \sum_{s=1}^{\tau}
        q_{\tau,s}\|y_s-\bar y_\tau\|_2^2
    \right]
    \le
    \left(2+\frac{4(1-a)}{a^2}\right)\eta^2.
\]
\end{lemma}

\begin{proof}
For each \(t\), define $\bar x_t:=\sum_{s=1}^tq_{t,s}x_s.$
The weighted-variance identity states that \(\forall z\in\RR^d\),
\begin{equation}
\label{eq:detailed_ns_weighted_variance_identity}
    \sum_{s=1}^tq_{t,s}\|y_s-z\|_2^2
    =
    \sum_{s=1}^tq_{t,s}\|y_s-\bar y_t\|_2^2
    +\|z-\bar y_t\|_2^2.
\end{equation}
Indeed, expand
\(y_s-z=(y_s-\bar y_t)+(\bar y_t-z)\); the cross term vanishes because
\(\sum_s q_{t,s}(y_s-\bar y_t)=0\), and
\(\sum_s q_{t,s}=1\).  Thus \(\bar y_t\) minimizes the weighted sum of
squared distances.  Choosing \(z=\bar x_t\) in
\eqref{eq:detailed_ns_weighted_variance_identity} gives
\begin{align}
    \sum_{s=1}^tq_{t,s}\|y_s-\bar y_t\|_2^2
    &\le
    \sum_{s=1}^tq_{t,s}\|y_s-\bar x_t\|_2^2
    =
    \sum_{s=1}^tq_{t,s}
    \|(y_s-x_s)+(x_s-\bar x_t)\|_2^2
    \notag\\
    &\le
    2\sum_{s=1}^tq_{t,s}\|y_s-x_s\|_2^2
    +
    2\sum_{s=1}^tq_{t,s}\|x_s-\bar x_t\|_2^2.
\label{eq:detailed_ns_variance_decomposition}
\end{align}
The last inequality uses \(\|u+v\|_2^2\le2\|u\|_2^2+2\|v\|_2^2\).

We first bound the first term in
\eqref{eq:detailed_ns_variance_decomposition}.  Since
\(y_s=x_s+r_s\delta_s\) and \(0\le r_s\le1\),
\begin{equation}
\label{eq:detailed_ns_segment_displacement}
    \|y_s-x_s\|_2^2
    =r_s^2\|\delta_s\|_2^2
    \le\|\delta_s\|_2^2.
\end{equation}
For a fixed \(s\), by direct substitution of the definitions of
\(\Pr(\tau=t)\) and \(q_{t,s}\), we have
\begin{align}
    \sum_{t=s}^T\Pr(\tau=t)q_{t,s}
    &=
    \frac1T\sum_{t=s}^{T-1}ab^{t-s}
    +\frac1T b^{T-s}
    =
    \frac1T\left(
        a\sum_{j=0}^{T-s-1}b^j+b^{T-s}
    \right)
    \notag\\
    &=
    \frac1T\left(1-b^{T-s}+b^{T-s}\right)
    =\frac1T.
\label{eq:detailed_ns_weight_marginal}
\end{align}
Consequently, conditioning first on the entire trajectory and then averaging
over \(\tau\), equations
\eqref{eq:detailed_ns_segment_displacement} and
\eqref{eq:detailed_ns_weight_marginal} imply
\begin{align}
    &\EE\left[
        2\sum_{s=1}^{\tau}
        q_{\tau,s}\|y_s-x_s\|_2^2
    \right]
    \le
    2\EE\left[
        \sum_{t=1}^T\Pr(\tau=t)
        \sum_{s=1}^tq_{t,s}\|\delta_s\|_2^2
    \right]
    \notag\\
    &\ =
    2\EE\left[
        \sum_{s=1}^T\|\delta_s\|_2^2
        \sum_{t=s}^T\Pr(\tau=t)q_{t,s}
    \right]
    =
    \frac2T\sum_{s=1}^T\EE[\|\delta_s\|_2^2].
\label{eq:detailed_ns_first_variance_term}
\end{align}
The change in summation order is valid because all sums are finite; a fixed
\(\delta_s\) appears precisely for output indices \(t\ge s\).

We next bound the second term in
\eqref{eq:detailed_ns_variance_decomposition}.  The same weighted-variance
identity applied to \(x_1,\ldots,x_t\) shows that \(\bar x_t\) minimizes
\(z\mapsto\sum_{s=1}^tq_{t,s}\|x_s-z\|_2^2\).  Taking \(z=x_t\) therefore
gives
\begin{equation}
\label{eq:detailed_ns_bar_x_optimality}
    \sum_{s=1}^tq_{t,s}\|x_s-\bar x_t\|_2^2
    \le
    \sum_{s=1}^tq_{t,s}\|x_s-x_t\|_2^2.
\end{equation}
For \(s<t\), the iterate recursion gives
\[
    x_t-x_s
    =\sum_{k=s}^{t-1}(x_{k+1}-x_k)
    =\sum_{k=s}^{t-1}\delta_k.
\]
We apply the Cauchy-Schwarz inequality to obtain
\begin{equation}
\label{eq:detailed_ns_path_cauchy}
    \|x_t-x_s\|_2^2\le\left(\sum_{k=s}^{t-1}\|\delta_k\|_2\right)^2
    \le
    (t-s)\sum_{k=s}^{t-1}\|\delta_k\|_2^2.
\end{equation}
Insert \eqref{eq:detailed_ns_path_cauchy} into the output expectation of the
right-hand side of \eqref{eq:detailed_ns_bar_x_optimality}.  Since the term
with \(s=t\) is zero, we obtain
\begin{align}
    &\sum_{t=1}^T\Pr(\tau=t)
      \sum_{s=1}^tq_{t,s}\|x_t-x_s\|_2^2
    \le
    \sum_{t=1}^T\Pr(\tau=t)
      \sum_{s=1}^{t-1}q_{t,s}(t-s)
      \sum_{k=s}^{t-1}\|\delta_k\|_2^2
    =
    \sum_{k=1}^{T-1}C_k\|\delta_k\|_2^2,
\label{eq:detailed_ns_path_collect}
\end{align}
where
\begin{equation}
\label{eq:detailed_ns_Ck_definition}
    C_k
    :=
    \sum_{t=k+1}^T\sum_{s=1}^k
    \Pr(\tau=t)q_{t,s}(t-s).
\end{equation}
To justify the last equality in
\eqref{eq:detailed_ns_path_collect}, we fix \(k\) and notice that the quantity
\(\|\delta_k\|_2^2\) occurs in the sum $\sum_{j=s}^{t-1}\|\delta_j\|_2^2$
exactly when \(s\le k\le t-1\), which is equivalent to \(s\le k\) and
\(t\ge k+1\). Hence its coefficient is exactly the double sum in
\eqref{eq:detailed_ns_Ck_definition}.

It remains to bound \(C_k\).  For \(t<T\),
\[
    \Pr(\tau=t)q_{t,s}=\frac{ab^{t-s}}T,
\]
whereas for \(t=T\),
\[
    \Pr(\tau=T)q_{T,s}=\frac{b^{T-s}}T.
\]
The exceptional terminal coefficient can be dominated by a geometric tail.
Indeed, setting \(n:=T-s\) and using \(1-b=a\),
\begin{align}
    \sum_{t=T}^{\infty}ab^{t-s}(t-s)
    &=
    ab^n\sum_{j=0}^{\infty}b^j(n+j)
    =
    ab^n
    \left(
        n\sum_{j=0}^{\infty}b^j
        +\sum_{j=0}^{\infty}jb^j
    \right)
    =
    ab^n\left(\frac n a+\frac b{a^2}\right)
    =
    b^n\left(n+\frac ba\right)
    \ge b^nn.
\label{eq:detailed_ns_terminal_tail}
\end{align}
Thus
\[
    b^{T-s}(T-s)
    \le
    \sum_{t=T}^{\infty}ab^{t-s}(t-s).
\]
The ordinary terms with \(t<T\), together with this upper bound for the
terminal term, imply
\begin{align}
    C_k
    &\le
    \frac aT
    \sum_{t=k+1}^{\infty}\sum_{s=1}^k
    b^{t-s}(t-s).
\label{eq:detailed_ns_Ck_tail}
\end{align}
Set $i:=t-k\ge1,
    \ 
    j:=k-s\in\{0,\ldots,k-1\}.$
Then \(t-s=i+j\), so
\begin{align}
    C_k
    &\le
    \frac aT
    \sum_{i=1}^{\infty}\sum_{j=0}^{k-1}
    b^{i+j}(i+j)
    \le
    \frac aT
    \sum_{i=1}^{\infty}\sum_{j=0}^{\infty}
    b^{i+j}(i+j).
\label{eq:detailed_ns_Ck_reindex}
\end{align}
The second inequality is valid because every summand is nonnegative. By the property of
geometric-series, we have
\[
    \sum_{j=0}^{\infty}b^j=\frac1{1-b}=\frac1a,
    \ 
    \sum_{j=0}^{\infty}jb^j
    =b\frac{d}{db}\left(\frac1{1-b}\right)
    =\frac b{a^2}
\]
also imply
\[
    \sum_{i=1}^{\infty}b^i=\frac ba,
    \ 
    \sum_{i=1}^{\infty}ib^i=\frac b{a^2}.
\]
Splitting \(i+j\) into its two parts and factoring the product series gives
\begin{align}
    \sum_{i=1}^{\infty}\sum_{j=0}^{\infty}
    b^{i+j}(i+j)
    &=
    \left(\sum_{i=1}^{\infty}ib^i\right)
    \left(\sum_{j=0}^{\infty}b^j\right)
    +
    \left(\sum_{i=1}^{\infty}b^i\right)
    \left(\sum_{j=0}^{\infty}jb^j\right)
    =
    \frac b{a^2}\frac1a
    +
    \frac ba\frac b{a^2}
    =
    \frac{b(1+b)}{a^3}.
\label{eq:detailed_ns_double_geometric}
\end{align}
Substituting this identity into
\eqref{eq:detailed_ns_Ck_reindex} and using \(1+b\le2\) yields
\begin{equation}
\label{eq:detailed_ns_Ck_bound}
    C_k
    \le
    \frac{b(1+b)}{a^2T}
    \le
    \frac{2b}{a^2T}.
\end{equation}
Combining
\eqref{eq:detailed_ns_bar_x_optimality},
\eqref{eq:detailed_ns_path_collect}, and
\eqref{eq:detailed_ns_Ck_bound} and taking expectation over the
algorithmic trajectory, we have
\begin{equation}
\label{eq:detailed_ns_second_variance_term}
    \EE\left[
        2\sum_{s=1}^{\tau}
        q_{\tau,s}\|x_s-\bar x_\tau\|_2^2
    \right]
    \le
    \frac{4b}{a^2T}
    \sum_{k=1}^{T-1}\EE[\|\delta_k\|_2^2].
\end{equation}

Finally, the normalized direction satisfies
\[
    \|z_t\|_2
    =
    \frac{b\|m_t\|_2}
    {\sqrt{b^2\|m_t\|_2^2+\rho^2}}
    \le1.
\]
Therefore, using \(\|U_t\|_{\mathrm{op}}\le1\),
\begin{equation}
\label{eq:detailed_ns_delta_bound}
    \|\delta_t\|_2
    =\eta\|U_tz_t\|_2
    \le
    \eta\|U_t\|_{\mathrm{op}}\|z_t\|_2
    \le\eta.
\end{equation}
Combining
\eqref{eq:detailed_ns_variance_decomposition},
\eqref{eq:detailed_ns_first_variance_term}, and
\eqref{eq:detailed_ns_second_variance_term}, and then applying
\eqref{eq:detailed_ns_delta_bound}, gives
\begin{align}
    &\EE\left[
        \sum_{s=1}^{\tau}
        q_{\tau,s}\|y_s-\bar y_\tau\|_2^2
    \right]
    \le
    \frac2T\sum_{s=1}^T\EE[\|\delta_s\|_2^2]
    +
    \frac{4b}{a^2T}\sum_{s=1}^{T-1}\EE[\|\delta_s\|_2^2]
    \le
    2\eta^2+\frac{4b}{a^2}\eta^2
    =
    \left(2+\frac{4b}{a^2}\right)\eta^2.
\label{eq:detailed_ns_output_variance}
\end{align}
This proves the localization bound because \(b=1-a\).
\end{proof}

We now possess the two components required by the definition of local
stationarity.  Lemma~\ref{lem:nonsmooth_knownT_output_gradient} controls the
mean gradient of a discrete witness distribution, while
Lemma~\ref{lem:nonsmooth_knownT_output_localization} controls the second moment
of that distribution around its mean.  The following lemma combines these
components, but deliberately leaves the average tracker norm unsubstituted so
that the final theorem proof has a transparent modular form.

\begin{lemma}[From the weighted output to local stationarity]
\label{lem:nonsmooth_knownT_stationarity_witness}
The randomized output \(\bar y_\tau\) satisfies
\begin{align*}
    \EE[\|\nabla f(\bar y_\tau)\|^{[\lambda]}]
    \le{}&
    \frac1T\sum_{t=2}^T\EE[\|m_t\|_2]
    +\sigma\sqrt{\frac a{2-a}}
    +\frac L{aT}+
    \left(2+\frac{4(1-a)}{a^2}\right)\lambda\eta^2.
\end{align*}
\end{lemma}

\begin{proof}
For each fixed \(t\), consider the discrete probability distribution
\[
    p_t:=\sum_{s=1}^tq_{t,s}\,\delta_{y_s},
\]
where \(\delta_{y_s}\) denotes the point mass at \(y_s\).  By
\eqref{eq:detailed_ns_q_sum_one}, this is a distribution, and its
mean is
\[
    \EE_{Y\sim p_t}[Y]
    =\sum_{s=1}^tq_{t,s}y_s
    =\bar y_t.
\]
Therefore, \(p_t\) is feasible in the infimum defining
\(\|\nabla f(\bar y_t)\|^{[\lambda]}\).  Evaluating that infimum at this
particular feasible distribution gives the pathwise inequality
\begin{equation}
\label{eq:detailed_ns_stationarity_witness}
    \|\nabla f(\bar y_t)\|^{[\lambda]}
    \le
    \left\|\sum_{s=1}^tq_{t,s}g_s\right\|_2
    +
    \lambda\sum_{s=1}^tq_{t,s}
        \|y_s-\bar y_t\|_2^2.
\end{equation}
Set \(t=\tau\), take total expectation in
\eqref{eq:detailed_ns_stationarity_witness}, and apply
\eqref{eq:detailed_ns_output_gradient_bound} and
\eqref{eq:detailed_ns_output_variance}.  We obtain
\begin{align}
    \EE[\|\nabla f(\bar y_\tau)\|^{[\lambda]}]
    \le{}&
    \frac1T\sum_{t=2}^T\EE[\|m_t\|_2]
    +\sigma\sqrt{\frac a{2-a}}
    +\frac L{aT}
    +
    \left(2+\frac{4b}{a^2}\right)\lambda\eta^2.
\label{eq:detailed_ns_before_tracker_substitution}
\end{align}
Recalling \(b=1-a\) gives the claimed inequality.
\end{proof}

With all the lemmas above,  we are now ready to assemble them to prove
Theorem~\ref{thm:nonsmooth_known_T}; no additional probabilistic or
optimization argument is needed beyond the lemmas above.

\begin{proof}[Proof of Theorem~\ref{thm:nonsmooth_known_T}]
We first summarize the lemmas that we have proved.

Lemma~\ref{lem:nonsmooth_knownT_segment_progress} converts the cumulative OCO
loss into objective progress, and
Lemma~\ref{lem:nonsmooth_knownT_identity_correlation} transfers this progress
bound to the fixed identity comparator.  The potential argument in
Lemma~\ref{lem:nonsmooth_knownT_tracker_norm} consequently yields a bound on
the average tracker norm.

On the output side,
Lemma~\ref{lem:nonsmooth_knownT_tracker_noise} controls the difference between
the stochastic and true-gradient trackers,
Lemma~\ref{lem:nonsmooth_knownT_output_gradient} controls the resulting
weighted average gradient, and
Lemma~\ref{lem:nonsmooth_knownT_output_localization} controls the spatial
variance of the same witness distribution. 

These last two estimates are
combined in Lemma~\ref{lem:nonsmooth_knownT_stationarity_witness}, which gives
\[
    \EE[\|\nabla f(\bar y_\tau)\|^{[\lambda]}]
    \le
    \frac1T\sum_{t=2}^T\EE[\|m_t\|_2]
    +\sigma\sqrt{\frac a{2-a}}
    +\frac L{aT}
    +\left(2+\frac{4(1-a)}{a^2}\right)\lambda\eta^2.
\]
Substituting the tracker estimate from
Lemma~\ref{lem:nonsmooth_knownT_tracker_norm} into this display gives
\begin{align}
    \EE[\|\nabla f(\bar y_\tau)\|^{[\lambda]}]
    \le{}&
    \frac{2M}{\eta T}
    +\frac{2\mathscr R_T(\cA,I_d)}T
    +\rho
    +\frac{a(L^2+\sigma^2)}\rho
    +\sigma\sqrt{\frac a{2-a}}
    \notag\\
    &+
    \frac L{aT}
    +
    \left(2+\frac{4(1-a)}{a^2}\right)\lambda\eta^2,
\end{align}
which proves the first assertion of the theorem.

We now prove the specialized bound by evaluating each term explicitly.  Let
\[
    B:=\sqrt{L^2+\sigma^2},
    \ 
    \varepsilon:=\varepsilon_T
    =
    \left(\frac{B^4M^2\lambda}{T^2}\right)^{1/7}.
\]
Raising the definition of \(\varepsilon\) to the power \(7/2\) gives the
useful identity
\begin{equation}
\label{eq:detailed_ns_epsilon_identity}
    \varepsilon^{7/2}
    =
    \frac{B^2M\sqrt\lambda}{T}.
\end{equation}
The stated lower bound on \(T\) is equivalent to the condition required for
\(a\le1/2\).  Indeed,
\begin{align*}
    \varepsilon\le\frac B{\sqrt2}
    &\iff
    \frac{B^4M^2\lambda}{T^2}
    \le\frac{B^7}{2^{7/2}}\iff
    T^2\ge\frac{2^{7/2}M^2\lambda}{B^3}\iff
    T\ge\frac{2^{7/4}M\sqrt\lambda}{B^{3/2}}.
\end{align*}
Thus, under the theorem's condition on \(T\),
\[
    a=\frac{\varepsilon^2}{B^2}\le\frac12,
\]
so the selected \(a\) is admissible.

We evaluate every non-regret term in the first bound.  Using
\eqref{eq:detailed_ns_epsilon_identity} and
\(\eta=\varepsilon^{5/2}/(B^2\sqrt\lambda)\),
\begin{align}
    \frac{2M}{\eta T}
    &=
    \frac{2MB^2\sqrt\lambda}{T\varepsilon^{5/2}}
    =
    \frac{2\varepsilon^{7/2}}{\varepsilon^{5/2}}
    =2\varepsilon.
\label{eq:detailed_ns_progress_specialized}
\end{align}
Because \(aB^2=\varepsilon^2\) and \(\rho=\varepsilon\),
\begin{equation}
\label{eq:detailed_ns_rho_specialized}
    \rho+\frac{a(L^2+\sigma^2)}\rho
    =
    \varepsilon+\frac{aB^2}{\varepsilon}
    =
    \varepsilon+\varepsilon
    =2\varepsilon.
\end{equation}
Moreover, since \(2-a\ge1\), \(\sigma\le B\), and
\(\sqrt a=\varepsilon/B\),
\begin{equation}
\label{eq:detailed_ns_noise_specialized}
    \sigma\sqrt{\frac a{2-a}}
    \le
    \sigma\sqrt a
    \le
    B\frac\varepsilon B
    =\varepsilon.
\end{equation}
For the terminal normalization term, solve
\eqref{eq:detailed_ns_epsilon_identity} for \(1/T\):
\[
    \frac1T
    =
    \frac{\varepsilon^{7/2}}{B^2M\sqrt\lambda}.
\]
Since \(a=\varepsilon^2/B^2\),
\begin{align}
    \frac L{aT}
    &=
    \frac{LB^2}{\varepsilon^2T}
    =
    \frac{LB^2}{\varepsilon^2}
    \frac{\varepsilon^{7/2}}{B^2M\sqrt\lambda}
    =
    \frac L{M\sqrt\lambda}\varepsilon^{3/2}.
\label{eq:detailed_ns_terminal_specialized}
\end{align}
Finally, since \(a\le1\),
\begin{align}
    2+\frac{4(1-a)}{a^2}
    &\le
    \frac2{a^2}+\frac4{a^2}
    =\frac6{a^2}.
\label{eq:detailed_ns_localization_coefficient}
\end{align}
The parameter choices give
\begin{align}
    \frac{\lambda\eta^2}{a^2}
    &=
    \frac{\lambda
        \bigl(\varepsilon^5/(B^4\lambda)\bigr)}
        {\varepsilon^4/B^4}
    =\varepsilon.
\label{eq:detailed_ns_localization_balance}
\end{align}
Combining
\eqref{eq:detailed_ns_localization_coefficient} and
\eqref{eq:detailed_ns_localization_balance} therefore gives
\begin{equation}
\label{eq:detailed_ns_localization_specialized}
    \left(2+\frac{4(1-a)}{a^2}\right)\lambda\eta^2
    \le6\varepsilon.
\end{equation}
Substituting
\eqref{eq:detailed_ns_progress_specialized},
\eqref{eq:detailed_ns_rho_specialized},
\eqref{eq:detailed_ns_noise_specialized},
\eqref{eq:detailed_ns_terminal_specialized}, and
\eqref{eq:detailed_ns_localization_specialized} into the first assertion
gives
\begin{align*}
    \EE[\|\nabla f(\bar y_\tau)\|^{[\lambda]}]
    &\le
    (2+2+1+6)\varepsilon
    +\frac L{M\sqrt\lambda}\varepsilon^{3/2}
    +\frac{2\mathscr R_T(\cA,I_d)}T\\
    &=
    11\varepsilon_T
    +\frac L{M\sqrt\lambda}\varepsilon_T^{3/2}
    +\frac{2\mathscr R_T(\cA,I_d)}T.
\end{align*}
This proves the second assertion.
\end{proof}

\section{Proofs in Section \ref{sec:lower_bound}}\label{app:proof_lower_bound}
The proofs of Theorems~\ref{thm:dichotomy} and~\ref{thm:tightness} rely on one key lemma, an oracle lower bound for stochastic first-order methods on instances with bounded gradients. We first state and prove this key lemma, and we then prove the two theorems. The proof of the key lemma uses four auxiliary lemmas, whose proofs are deferred to the end of this section.

A randomized algorithm chooses its $k$-th query point $z_k$ as a measurable function of its random seed $\omega$ and of the answers $\gamma_1,\ldots,\gamma_{k-1}$ of $\cO$ at $z_1,\ldots,z_{k-1}$, and every answer uses fresh randomness.

\begin{lemma}\label{lemma:oracle_lower_bound}
Let $M,\beta,\sigma>0$ and $0<\varepsilon\le 2^{-10}\min\cbr{\sigma,\sqrt{M\beta}}$, and let $N$ be a positive integer with $N\le 2^{-36}M\beta\sigma^2\varepsilon^{-4}$. There is a dimension $d$, depending only on $M$, $\beta$, $\varepsilon$, and $N$, such that the following holds. For every randomized algorithm on $\RR^d$ that queries $\cO$ once per step, there exists an instance $(f,\cO)$ on $\RR^d$, with $\|\nabla f(x)\|_2\le\sqrt{M\beta}$ for every $x\in\RR^d$, on which, with probability at least $\frac34$, the query points satisfy $\|\nabla f(z_k)\|_2>2\varepsilon$ for every $k\in[N]$. In particular, $\EE[\|\nabla f(z_k)\|_2^2]\ge 3\varepsilon^2$ for every $k\in[N]$.
\end{lemma}

The idea of the proof is as follows. The objective depends on $x$ through $J$ correlations $v_j(x)$, $j\in[J]$, between blocks of coordinates of $x$ and random sign vectors, and its gradient norm exceeds $2\varepsilon$ unless $|v_j(x)|>1$ for every $j$. At a query point $x$, the oracle answer depends on the sign vector with index $\pi(x)+1$ only with probability $p$, where $\pi(x)$ is the largest index $j$ with $|v_j(x)|>\frac12$. A correlation with a sign vector on which no previous answer depends is small with high probability. Hence, $|v_J|$ can exceed $\frac12$ at a query point only after about $J/p$ queries, and $J/p$ is larger than $N$.

\begin{proof}[Proof of Lemma \ref{lemma:oracle_lower_bound}]
The proof has five steps. Steps~1 to~3 construct the objective and the oracle. Step~4 chooses the parameters and shows that the rescaled objective and oracle form an instance. Step~5 shows that, on this instance, $\|\nabla f(z_k)\|_2>2\varepsilon$ for every $k\in[N]$ with probability at least $\frac34$.

\textbf{Step 1.} We first construct a smooth function $h:\RR^J\to\RR$ with $\|\nabla h(v)\|_\infty\ge 1$ whenever $|v_j|\le 1$ for some $j\in[J]$. We use the smooth zero-chain of \citet{carmon2020lower}. Let
\[
\psi(t)=\begin{cases}0, & t\le\frac12,\\ \exp\big(1-(2t-1)^{-2}\big), & t>\frac12,\end{cases}\quad \varphi(t)=\sqrt e\int_{-\infty}^{t}e^{-\tau^2/2}d\tau.
\]
Both functions are infinitely differentiable, $\psi$ is nondecreasing, $\psi(1)=1$, and $\psi=\psi'=0$ on $(-\infty,\frac12]$. Moreover, $0\le\psi<e$, $0\le\psi'\le\sqrt{54/e}$, $0<\varphi<\sqrt{2\pi e}$, and $0<\varphi'\le\sqrt e$ \citep{carmon2020lower}, and a direct computation gives $|\psi''|\le 32.4$ and $|\varphi''|\le 1$. For $J\ge 1$, define $h:\RR^J\to\RR$ by
\[
h(v)=-\varphi(v_1)+\sum_{j=2}^{J}\big(\psi(-v_{j-1})\varphi(-v_j)-\psi(v_{j-1})\varphi(v_j)\big).
\]
Its partial derivatives are
\begin{align}\label{eq:lb_chain_grad}
\partial_jh(v)=-\psi(-v_{j-1})\varphi'(-v_j)-\psi(v_{j-1})\varphi'(v_j)-\psi'(-v_j)\varphi(-v_{j+1})-\psi'(v_j)\varphi(v_{j+1}),
\end{align}
where the first two terms are replaced by $-\varphi'(v_1)$ when $j=1$, and the last two terms are absent when $j=J$. We call the first two terms in \eqref{eq:lb_chain_grad} the first pair, and the last two terms the second pair. All four terms are nonpositive. Since $\psi$ and $\psi'$ vanish on $(-\infty,\frac12]$, we have $\psi(t)\psi(-t)=\psi'(t)\psi'(-t)=0$ for every $t$, so at most one term of each pair is nonzero. For the same reason, the first pair vanishes when $|v_{j-1}|\le\frac12$, and the second pair vanishes when $|v_j|\le\frac12$. For $k\in\cbr{0,\ldots,J}$ and $v\in\RR^J$, let $v^{(k)}=(v_1,\ldots,v_k,0,\ldots,0)$ be the vector obtained from $v$ by setting its last $J-k$ entries to zero. We will use the following properties of $h$.
\begin{lemma}\label{lemma:lb_chain}
For every $J\ge 1$ and $v\in\RR^J$, the following statements hold.
\begin{itemize}
    \item[(a)] $h(0)-\inf h\le 12J$.
    \item[(b)] $|\partial_jh(v)|\le 23$ for every $j$, and $\nabla h$ is $152$-Lipschitz.
    \item[(c)] Let $0\le k<J$. If $|v_j|\le\frac12$ for every $j>k$, then $\nabla h(v)=\nabla h(v^{(k+1)})$, and $\partial_jh(v)=0$ for every $j\ge k+2$.
    \item[(d)] If $j$ is the smallest index with $|v_j|\le 1$, then $\partial_jh(v)\le -1$.
\end{itemize}
\end{lemma}

\textbf{Step 2.} We next compose $h$ with correlations between blocks of coordinates and random sign vectors, following \citet[Section~6]{jin2026tight}. Let $n\ge 1$, and index the coordinates of $x\in\RR^{nJ}$ by pairs $(j,\ell)\in[J]\times[n]$, so that the $j$-th block of $x$ is $(x_{j1},\ldots,x_{jn})$. Fix signs $\theta\in\cbr{-1,1}^{J\times n}$, and let
\[
r_{j\ell}(x)=\theta_{j\ell}\tanh\big(\sqrt n\,x_{j\ell}\big)\in(-1,1).
\]
For $j\in[J]$, let $v_j(x)$ be the correlation between the $j$-th block of $x$ and the $j$-th row of $\theta$, and let $F$ be the sum of $h(v(x))$ and a $\log\cosh$ term:
\[
v_j(x)=\frac4n\sum_{\ell=1}^{n}r_{j\ell}(x),\qquad F(x)=h\big(v(x)\big)+\frac8n\sum_{j=1}^{J}\sum_{\ell=1}^{n}\log\cosh\big(\sqrt n\,x_{j\ell}\big).
\]
Since $\theta_{j\ell}^2=1$ and $\frac{d}{ds}\tanh(s)=1-\tanh^2(s)$, we have
\begin{align}\label{eq:lb_grad}
\frac{\partial F}{\partial x_{j\ell}}(x)=\frac{\theta_{j\ell}}{\sqrt n}\Big(4\,\partial_jh\big(v(x)\big)\big(1-r_{j\ell}(x)^2\big)+8\,r_{j\ell}(x)\Big).
\end{align}
The first term in \eqref{eq:lb_grad} comes from $h(v(x))$, and the second term comes from the $\log\cosh$ term. When $|x_{j\ell}|$ is large, $1-r_{j\ell}(x)^2$ is close to zero, so the first term is small, while the second term has absolute value close to $8/\sqrt n$. To show that $F$ is smooth, that its gradient is bounded, and that its gradient norm exceeds $\frac12$ whenever $|v_j(x)|\le 1$ for some $j$, we have the following lemma.
\begin{lemma}\label{lemma:lb_encoded}
For every $\theta$, the following statements hold.
\begin{itemize}
    \item[(a)] $F$ is $2511$-smooth, and $F(0)-\inf F\le 12J$.
    \item[(b)] $\|\nabla F(x)\|_2\le 100\sqrt J$ for every $x$.
    \item[(c)] If $|v_j(x)|\le 1$ for some $j\in[J]$, then $\|\nabla F(x)\|_2>\frac12$.
\end{itemize}
\end{lemma}

\textbf{Step 3.} We now define a stochastic gradient oracle for $F$. For $k\in\cbr{0,\ldots,J}$, let $F_k$ be the function obtained from $F$ by replacing $v(x)$ with $v^{(k)}(x)$,
\[
F_k(x)=h\big(v^{(k)}(x)\big)+\frac8n\sum_{j=1}^{J}\sum_{\ell=1}^{n}\log\cosh\big(\sqrt n\,x_{j\ell}\big),
\]
so that $F_J=F$. We also set $F_{J+1}=F$. Since $v_j$ depends on $\theta$ only through its $j$-th row, the gradient $\nabla F_k$ depends on $\theta$ only through its first $k$ rows. For $x\in\RR^{nJ}$, let
\[
\pi(x)=\max\cbr{j\in[J]:|v_j(x)|>\tfrac12},
\]
with $\pi(x)=0$ if $|v_j(x)|\le\frac12$ for every $j\in[J]$. Following the probability-$p$ zero-chain oracle of \citet{arjevani2023lower}, for $p\in(0,1]$, the oracle answers a query at $x$ with
\[
\Gamma(x;\zeta)=\nabla F_{\pi(x)}(x)+\frac{\zeta}{p}\Big(\nabla F_{\pi(x)+1}(x)-\nabla F_{\pi(x)}(x)\Big),
\]
where $\zeta\sim{\rm Bernoulli}(p)$ is drawn independently at every query. Hence, the answer $\Gamma(x;\zeta)$ depends on $\theta$ only through its first $\pi(x)$ rows if $\zeta=0$, and only through its first $\min\cbr{\pi(x)+1,J}$ rows if $\zeta=1$. To show that $\Gamma$ is unbiased and has bounded variance, we have the following lemma.
\begin{lemma}\label{lemma:lb_oracle}
For every $x$, we have $\EE_\zeta[\Gamma(x;\zeta)]=\nabla F(x)$ and $\EE_\zeta[\|\Gamma(x;\zeta)-\nabla F(x)\|_2^2]\le\frac{540}{p}$.
\end{lemma}

\textbf{Step 4.} We now choose the parameters and rescale $F$ and $\Gamma$, so that the resulting pair satisfies Assumptions~\ref{ass:nonconvex_function} and~\ref{ass:gradient_oracle}. Translating the instance, we assume that the initial point is $x_1=0$. Set
\[
C_0=2511,\  \delta=\frac{4C_0\varepsilon}{\beta},\  A=\frac{16C_0\varepsilon^2}{\beta},\  J=\left\lfloor\frac{M\beta}{192C_0\varepsilon^2}\right\rfloor,\  p=\frac{8640\varepsilon^2}{\sigma^2},\  n=\big\lceil 128\log(16NJ)\big\rceil,
\]
and $d=nJ$, which depends only on $M$, $\beta$, $\varepsilon$, and $N$. Since $\varepsilon\le 2^{-10}\min\cbr{\sigma,\sqrt{M\beta}}$ and $192C_0\le 2^{20}$, we have $J\ge 1$ and $p\le 1$. For every $\theta\in\cbr{-1,1}^{J\times n}$, consider
\[
f(x)=A\,F(x/\delta),\qquad \cO(x)=\frac{A}{\delta}\,\Gamma(x/\delta;\zeta).
\]
We have $A/\delta^2=\beta/C_0$ and $A/\delta=4\varepsilon$. By Lemma~\ref{lemma:lb_encoded}(a), $f$ is $\beta$-smooth and $f(0)-\inf f\le 12AJ\le M$. By Lemma~\ref{lemma:lb_encoded}(b), $\|\nabla f(x)\|_2\le 400\varepsilon\sqrt J\le 400\sqrt{M\beta}/\sqrt{192C_0}<\sqrt{M\beta}$. By Lemma~\ref{lemma:lb_oracle}, $\cO$ is unbiased, and its variance is at most $16\varepsilon^2\cdot 540/p=\sigma^2$. Hence $(f,\cO)$ is an instance for every $\theta$. Moreover, by Lemma~\ref{lemma:lb_encoded}(c),
\begin{align}\label{eq:lb_large_grad}
\|\nabla f(x)\|_2=4\varepsilon\|\nabla F(x/\delta)\|_2>2\varepsilon\quad\text{whenever }|v_j(x/\delta)|\le 1\text{ for some }j\in[J].
\end{align}

\textbf{Step 5.} Finally, we show that, for a suitable $\theta$, with probability at least $\frac34$, every query point satisfies $|v_J(z_k/\delta)|\le\frac12$, and hence $\|\nabla f(z_k)\|_2>2\varepsilon$ by \eqref{eq:lb_large_grad}. Fix a randomized algorithm on $\RR^d$. Querying $\cO$ at $z$ amounts to querying $\Gamma$ at $z/\delta$ and multiplying the answer by $A/\delta$. Hence $z_1/\delta,\ldots,z_N/\delta$ are the query points of a randomized algorithm that queries $\Gamma$. We now draw $\theta$ uniformly at random, independently of the algorithm and the oracle. To bound the probability that $|v_J|$ exceeds $\frac12$ at some query point, we have the following lemma.
\begin{lemma}\label{lemma:lb_hidden}
Consider a randomized algorithm that queries $\Gamma$. Let $\omega$ be its random seed, and let $\zeta_1,\ldots,\zeta_N$ be the Bernoulli variables that $\Gamma$ uses at its first $N$ queries. Let $\theta$ be uniform on $\cbr{-1,1}^{J\times n}$, independent of $\omega$ and $\zeta_1,\ldots,\zeta_N$. Then, with probability at least $1-2NJe^{-n/128}-Np/J$, its query points $z_1,\ldots,z_N$ satisfy $|v_J(z_k)|\le\frac12$ for every $k\in[N]$.
\end{lemma}
By the choice of $n$, we have $2NJe^{-n/128}\le\frac18$. Since $\lfloor s\rfloor\ge s/2$ for $s\ge 1$, we have $J\ge\frac{M\beta}{384C_0\varepsilon^2}$, and hence
\[
\frac{Np}{J}\le\frac{384\cdot8640\cdot C_0\,N\varepsilon^4}{M\beta\sigma^2}\le\frac{384\cdot8640\cdot2511}{2^{36}}\le\frac18.
\]
By Lemma~\ref{lemma:lb_hidden}, with probability at least $\frac34$, we have $|v_J(z_k/\delta)|\le\frac12$ for every $k\in[N]$, and hence $\|\nabla f(z_k)\|_2>2\varepsilon$ for every $k\in[N]$ by \eqref{eq:lb_large_grad}. This probability is an average over $\theta$. Therefore, we can fix $\theta$ such that the same holds with probability at least $\frac34$ over the randomness of the algorithm and the oracle alone. For this instance, $\EE[\|\nabla f(z_k)\|_2^2]\ge 4\varepsilon^2\cdot\frac34=3\varepsilon^2$ for every $k\in[N]$. We finish the proof.
\end{proof}

Equipped with Lemma~\ref{lemma:oracle_lower_bound}, we now prove Theorems~\ref{thm:dichotomy} and~\ref{thm:tightness}.
\begin{proof}[Proof of Theorem \ref{thm:dichotomy}]
Fix an OCO algorithm $\cA$. Consider the algorithm $\cP$ that runs $\cC$ with $\cA$ and, in each round, makes one additional query at the reported iterate $x_t$, discards the answer, and pads the round with dummy queries so that it contains exactly $q+1$ queries. The additional queries do not change the behavior of $\cC$. Thus, $\cP$ is a randomized algorithm that queries $\cO$ once per step for $N=(q+1)T$ steps, and every reported iterate $x_t$ is one of its query points.

Set $\varepsilon^4=2^{-36}M\beta\sigma^2/N$. Then $\varepsilon\le 2^{-10}\min\cbr{\sigma,\sqrt{M\beta}}$ is equivalent to the condition $N\ge 16\max\cbr{\frac{M\beta}{\sigma^2},\frac{\sigma^2}{M\beta}}$ of the theorem. Applying Lemma~\ref{lemma:oracle_lower_bound} to $\cP$, we obtain an instance on which $\EE[\|\nabla f(x_t)\|_2^2]\ge 3\varepsilon^2$ for every $t\in[T]$. With $c=2^{-18}$, we have $\varepsilon^2=c\,\sigma\sqrt{M\beta}/\sqrt N$. Combining this with the guarantee \eqref{eq:regret_controlled} on the same instance, we obtain
\begin{align}\label{eq:dichotomy_key}
b_T+\alpha_T\frac{\mathscr R_T(\cA,u)}{T}\ge\frac{1}{T}\sum_{t=1}^{T}\EE[\|\nabla f(x_t)\|_2^2]\ge 3c\,\sigma\frac{\sqrt{M\beta}}{\sqrt N}.
\end{align}
If (i) fails, then \eqref{eq:dichotomy_key} gives $\alpha_T\mathscr R_T(\cA,u)/T>2c\,\sigma\sqrt{M\beta}/\sqrt N>0$. Since $\alpha_T\ge 0$, we have $\alpha_T>0$, and on this instance,
\[
\mathscr R_T(\cA,u)>2c\,\frac{\sqrt{M\beta}}{\alpha_T}\,\sigma\frac{T}{\sqrt N}\ge c\,\frac{\sqrt{M\beta}}{\alpha_T}\,\sigma\sqrt{\frac{T}{q+1}}.
\]
Since $\cA$ is arbitrary, (ii) holds. We finish the proof.
\end{proof}

We now prove Theorem~\ref{thm:tightness}. The proof combines two lower bounds on the average squared gradient norm, both on the same hard instance. With $c=2^{-18}$, the first is $3c\,\sigma\sqrt{M\beta}/\sqrt T$, which follows from Lemma~\ref{lemma:oracle_lower_bound}, because Algorithm~\ref{alg:alg_known_T} queries $\cO$ once per round, at $x_t$. The second is $2^7c\sqrt{M\beta}\,\mathscr R_T(\cA,I_d)/T$. It holds because, in expectation, the loss $L_t$ is linear in the true gradient $\nabla f(x_t)$, so the regret of $\cA$ is controlled by the gradient norms along the trajectory. The average squared gradient norm is at least the maximum of the two bounds, and $\max\cbr{x,y}\ge\frac{x+y}{2}$ gives the theorem.
\begin{proof}[Proof of Theorem \ref{thm:tightness}]
We use the notation of Appendix~\ref{app:proof_thm}:
\[
G=\sqrt{M\beta},\ g_t=\nabla f(x_t),\ e_t=m_t-g_t,\ S=\sum_{t=1}^{T}\EE[\|g_t\|^2],\ E_{tr}=\sum_{t=1}^{T}\EE[\|e_t\|^2],\ V_T=\sum_{t=1}^{T}\EE[\|s_t-g_t\|^2].
\]
We prove the theorem with $c=2^{-18}$. Since $T\ge 16G^2/\sigma^2$, we have $\sigma\sqrt T\ge 4G$, and hence $a=G/(\sigma\sqrt T)$.

We first bound $S$ from below on a hard instance. Algorithm~\ref{alg:alg_known_T}, run with $\cA$, is a randomized algorithm that queries $\cO$ once per round, at $x_t$. Indeed, $x_t$ is a measurable function of the random seed of $\cA$ and of $s_1,\ldots,s_{t-1}$. Set $\varepsilon^4=2^{-36}G^2\sigma^2/T$. The condition on $T$ is equivalent to $\varepsilon\le 2^{-10}\min\cbr{\sigma,G}$, and $N=T$ satisfies $N\le 2^{-36}G^2\sigma^2\varepsilon^{-4}$. By Lemma~\ref{lemma:oracle_lower_bound}, there exists an instance on which $\EE[\|g_t\|^2]\ge 3\varepsilon^2$ for every $t\in[T]$. Since $\varepsilon^2=c\,G\sigma/\sqrt T$, we obtain
\begin{align}\label{eq:tight_S}
S\ge 3c\,G\sigma\sqrt T.
\end{align}

We next bound the regret from above on this instance. Conditioned on $\cH_t$, the variables $U_t$ and $m_t$ are fixed, and $\EE[s_t|\cH_t]=g_t$. Hence
\[
G\mathscr R_T(\cA,I_d)=\sum_{t=1}^{T}\EE[\inner{s_t}{(I_d-U_t)m_t}]=\sum_{t=1}^{T}\EE[\inner{g_t}{(I_d-U_t)m_t}].
\]
Since $\|U_t\|_{op}\le 1$, we have $\|I_d-U_t\|_{op}\le 2$. Using $\|m_t\|\le\|g_t\|+\|e_t\|$ and the Cauchy-Schwarz inequality, we obtain
\[
G\mathscr R_T(\cA,I_d)\le 2\sum_{t=1}^{T}\EE[\|g_t\|^2+\|g_t\|\|e_t\|]\le 2S+2\sqrt{SE_{tr}}.
\]
By Lemma~\ref{lemma:import1}, $V_T\le\sigma^2T$, and $a=G/(\sigma\sqrt T)$, we have $E_{tr}\le 9G\sigma\sqrt T+\frac{1}{128}S$. By \eqref{eq:tight_S}, $9G\sigma\sqrt T\le 3\cdot 2^{18}S$. Hence $E_{tr}\le(3\cdot 2^{18}+2^{-7})S\le 887^2S$, and
\begin{align}\label{eq:tight_R}
G\mathscr R_T(\cA,I_d)\le 2(1+887)S\le 2^{11}S.
\end{align}

We now combine the two bounds. Let $x=\sigma G/\sqrt T$ and $y=G\mathscr R_T(\cA,I_d)/T$. Since $2^{-11}=2^7c$, \eqref{eq:tight_S} and \eqref{eq:tight_R} give
\[
\frac{S}{T}\ge\max\cbr{3cx,2^7cy}\ge 2c\max\cbr{x,y}\ge c(x+y).
\]
The second inequality holds because $3cx\ge 2cx$, and $2^7cy\ge 2cy$ whenever $y>x>0$. The last inequality is $\max\cbr{x,y}\ge\frac{x+y}{2}$. We finish the proof.
\end{proof}

Finally, we provide the proofs of Lemmas~\ref{lemma:lb_chain}, \ref{lemma:lb_encoded}, \ref{lemma:lb_oracle}, and~\ref{lemma:lb_hidden}, which we used in the proof of Lemma~\ref{lemma:oracle_lower_bound}. Throughout, we use the notation introduced in that proof.
\begin{proof}[Proof of Lemma \ref{lemma:lb_chain}]
For (a), each summand of $h$ has at most one nonzero term, so it lies in $(-e\sqrt{2\pi e},e\sqrt{2\pi e})$. Since $h(0)=-\varphi(0)$, we obtain $h(0)-\inf h\le\sqrt{2\pi e}+(J-1)e\sqrt{2\pi e}\le 12J$.

For (b), \eqref{eq:lb_chain_grad} gives $|\partial_jh(v)|\le e\sqrt e+\sqrt{54/e}\sqrt{2\pi e}\le 23$. The Hessian of $h$ is tridiagonal. Its diagonal entries are bounded by $e+32.4\sqrt{2\pi e}$, and its off-diagonal entries are bounded by $\sqrt{54/e}\sqrt e=\sqrt{54}$. By the Gershgorin circle theorem, its operator norm is at most $e+32.4\sqrt{2\pi e}+2\sqrt{54}\le 152$.

For (c), let $w=v^{(k+1)}$. By \eqref{eq:lb_chain_grad}, $\partial_jh(v)$ depends only on $v_{j-1}$, $v_j$, and $v_{j+1}$. If $j\le k$, these entries are the same in $v$ and $w$, so $\partial_jh(v)=\partial_jh(w)$. If $j=k+1$, only the entry $v_{k+2}$ may differ, and it appears only in the second pair, which vanishes for both $v$ and $w$ because $|v_{k+1}|\le\frac12$. If $j\ge k+2$, both pairs vanish for both $v$ and $w$, because $|v_{j-1}|,|v_j|,|w_{j-1}|,|w_j|\le\frac12$.

For (d), if $j=1$, then $\partial_1h(v)\le-\varphi'(v_1)\le-\sqrt e\,e^{-1/2}=-1$. If $j\ge2$, then $|v_{j-1}|>1$, so one term of the first pair equals $-\psi(|v_{j-1}|)\varphi'(\pm v_j)$, where $\psi(|v_{j-1}|)\ge\psi(1)=1$ and $\varphi'(\pm v_j)\ge 1$ because $|v_j|\le 1$. Hence the first pair is at most $-1$, and the second pair is nonpositive. We finish the proof.
\end{proof}

\begin{proof}[Proof of Lemma \ref{lemma:lb_encoded}]
For (a), $v_j$ depends only on the $j$-th block of $x$, and $\|\nabla v_j(x)\|_2^2=\frac{16}{n}\sum_{\ell}(1-r_{j\ell}^2)^2\le 16$. Hence the Jacobian $Dv(x)$ has rows with disjoint supports and norms at most $4$, so $\|Dv(x)\|_{\rm op}\le 4$. The Hessian of $v_j$ is diagonal, with entries $-8r_{j\ell}(1-r_{j\ell}^2)$ of absolute value at most $\frac{16}{3\sqrt3}$. The Hessian of the $\log\cosh$ term is diagonal, with entries $8(1-r_{j\ell}^2)\le 8$. By the chain rule,
\[
\nabla^2F(x)=Dv(x)^\top\nabla^2h\big(v(x)\big)Dv(x)+\sum_{j=1}^{J}\partial_jh\big(v(x)\big)\nabla^2v_j(x)+8\,\mathrm{diag}\big(1-r_{j\ell}(x)^2\big).
\]
By Lemma~\ref{lemma:lb_chain}(b), $\|\nabla^2F(x)\|_{\rm op}\le 16\cdot152+23\cdot\frac{16}{3\sqrt3}+8\le 2511$. Since $v(0)=0$ and the $\log\cosh$ term is nonnegative and vanishes at $0$, Lemma~\ref{lemma:lb_chain}(a) gives $F(0)-\inf F\le h(0)-\inf h\le 12J$.

For (b), the gradient of $h(v(x))$ is $Dv(x)^\top\nabla h(v(x))$, whose norm is at most $4\cdot23\sqrt J$. The gradient of the $\log\cosh$ term has entries $8\tanh(\sqrt n\,x_{j\ell})/\sqrt n$, so its norm is at most $8\sqrt J$.

For (c), let $j$ be the smallest index with $|v_j(x)|\le 1$, and let $b=\partial_jh(v(x))$. By Lemma~\ref{lemma:lb_chain}(d), $b\le -1$. By \eqref{eq:lb_grad}, the $j$-th block of $\nabla F(x)$ satisfies
\[
\sum_{\ell=1}^{n}\left(\frac{\partial F}{\partial x_{j\ell}}(x)\right)^2=\frac1n\sum_{\ell=1}^{n}q\big(r_{j\ell}(x)\big)^2,\qquad q(r)=4b(1-r^2)+8r.
\]
The quadratic $q$ factors as $q(r)=4|b|(r-r_+)(r-r_-)$, with roots
\[
r_+=\frac{\sqrt{1+b^2}-1}{|b|}\ge\sqrt2-1,\qquad r_-=-\frac{\sqrt{1+b^2}+1}{|b|}.
\]
For $r\ge -1$, we have $r-r_-\ge -1-r_->\frac{1}{|b|}$, and hence $|q(r)|\ge 4|r-r_+|$. Moreover, $\frac1n\sum_{\ell}r_{j\ell}(x)=\frac{v_j(x)}{4}\le\frac14$. Using Jensen's inequality in the third step, we obtain
\begin{align*}
\|\nabla F(x)\|_2^2&\ge\frac1n\sum_{\ell=1}^{n}q\big(r_{j\ell}(x)\big)^2\ge\frac{16}{n}\sum_{\ell=1}^{n}\big(r_{j\ell}(x)-r_+\big)^2\\
&\ge 16\left(r_+-\frac1n\sum_{\ell=1}^{n}r_{j\ell}(x)\right)^2\ge 16\left(\sqrt2-1-\frac14\right)^2>\frac14.
\end{align*}
We finish the proof.
\end{proof}

\begin{proof}[Proof of Lemma \ref{lemma:lb_oracle}]
Let $k=\pi(x)$ and $v=v(x)$. If $k=J$, then $\Gamma(x;\zeta)=\nabla F(x)$, and there is nothing to prove. Let $k<J$. Then $|v_j|\le\frac12$ for every $j>k$.

We first show that $\nabla F(x)=\nabla F_{k+1}(x)$. By the chain rule,
\[
\nabla F(x)-\nabla F_{k+1}(x)=\sum_{j=1}^{k+1}\Big(\partial_jh(v)-\partial_jh\big(v^{(k+1)}\big)\Big)\nabla v_j(x)+\sum_{j=k+2}^{J}\partial_jh(v)\,\nabla v_j(x).
\]
Both sums vanish by Lemma~\ref{lemma:lb_chain}(c).

We next bound $\Delta=\nabla F_{k+1}(x)-\nabla F_k(x)$. By the chain rule,
\[
\Delta=\sum_{j=1}^{k}\Big(\partial_jh\big(v^{(k+1)}\big)-\partial_jh\big(v^{(k)}\big)\Big)\nabla v_j(x)+\partial_{k+1}h\big(v^{(k+1)}\big)\nabla v_{k+1}(x).
\]
The vectors $v^{(k+1)}$ and $v^{(k)}$ differ only in the entry $k+1$. By \eqref{eq:lb_chain_grad}, the terms with $j\le k-1$ vanish. For $j=k$, only the second pair in \eqref{eq:lb_chain_grad} changes, and at most one of its terms is nonzero. Since $|v_{k+1}|\le\frac12$ and $\varphi'\le\sqrt e$, this gives
\[
\Big|\partial_kh\big(v^{(k+1)}\big)-\partial_kh\big(v^{(k)}\big)\Big|\le\sqrt{\frac{54}{e}}\cdot\frac{\sqrt e}{2}=\frac{\sqrt{54}}{2}.
\]
For $j=k+1$, the second pair vanishes because $|v_{k+1}|\le\frac12$, so $|\partial_{k+1}h(v^{(k+1)})|\le e\sqrt e$. Since $\nabla v_k(x)$ and $\nabla v_{k+1}(x)$ have disjoint supports and norms at most $4$, we obtain
\[
\|\Delta\|_2^2\le 16\left(\frac{54}{4}+e^3\right)\le 540.
\]
Finally, $\Gamma(x;\zeta)-\nabla F(x)=\big(\frac{\zeta}{p}-1\big)\Delta$. Since $\EE[\frac{\zeta}{p}]=1$ and $\EE[(\frac{\zeta}{p}-1)^2]=\frac{1-p}{p}$, both claims follow. We finish the proof.
\end{proof}

\begin{proof}[Proof of Lemma \ref{lemma:lb_hidden}]
Write $z_k=Z_k(\omega,\gamma_1,\ldots,\gamma_{k-1})$ and $\gamma_k=\Gamma(z_k;\zeta_k)$. The query point $z_k$ may depend on $\theta$ through the answers $\gamma_1,\ldots,\gamma_{k-1}$. We therefore compare the actual interaction with a surrogate interaction, whose $k$-th answer depends on $\theta$ only through its first $\rho_k$ rows for a random index $\rho_k$ that we control.

We first define the surrogate interaction. For $\rho\in\cbr{0,\ldots,J}$, let $\pi_\rho(x)=\max\cbr{j\le\rho:|v_j(x)|>\frac12}$, with $\pi_\rho(x)=0$ if $|v_j(x)|\le\frac12$ for every $j\le\rho$, so that $\pi_J=\pi$. Set $\rho_0=0$. For $k=1,\ldots,N$, let $\tilde z_k=Z_k(\omega,\tilde\gamma_1,\ldots,\tilde\gamma_{k-1})$ and $\tilde\pi_k=\pi_{\rho_{k-1}}(\tilde z_k)$, and let
\[
\tilde\gamma_k=\nabla F_{\tilde\pi_k}(\tilde z_k)+\frac{\zeta_k}{p}\Big(\nabla F_{\tilde\pi_k+1}(\tilde z_k)-\nabla F_{\tilde\pi_k}(\tilde z_k)\Big).
\]
Let $\rho_k=\rho_{k-1}+1$ if $\tilde\pi_k=\rho_{k-1}<J$ and $\zeta_k=1$, and $\rho_k=\rho_{k-1}$ otherwise.

We next show that $\tilde\gamma_k$ depends on $\theta$ only through its first $\rho_k$ rows. By Step~3, $\tilde\gamma_k$ depends on $\theta$ only through its first $\tilde\pi_k$ rows if $\zeta_k=0$, and only through its first $\min\cbr{\tilde\pi_k+1,J}$ rows if $\zeta_k=1$. If $\zeta_k=0$, then $\tilde\pi_k\le\rho_{k-1}\le\rho_k$. If $\zeta_k=1$ and $\tilde\pi_k<\rho_{k-1}$, then $\tilde\pi_k+1\le\rho_{k-1}\le\rho_k$. If $\zeta_k=1$ and $\tilde\pi_k=\rho_{k-1}<J$, then $\rho_k=\tilde\pi_k+1$. If $\tilde\pi_k=\rho_{k-1}=J$, then $\rho_k=J$. This proves the claim. By induction on $k$, for every $i$, the indicator of the event $\cbr{\rho_k=i}$ and, on this event, the answers $\tilde\gamma_1,\ldots,\tilde\gamma_k$ are functions of $\omega$, $\zeta_1,\ldots,\zeta_k$, and the first $i$ rows of $\theta$. In particular, on the event $\cbr{\rho_{k-1}=i}$, the point $\tilde z_k=Z_k(\omega,\tilde\gamma_1,\ldots,\tilde\gamma_{k-1})$ is a function of $\omega$, $\zeta_1,\ldots,\zeta_{k-1}$, and the first $i$ rows of $\theta$.

We then show that, with high probability, $|v_j(\tilde z_k)|\le\frac12$ for every $k\le N$ and every $j>\rho_{k-1}$. Fix $k$, $i$, and $j>i$. Condition on $\omega$, $\zeta_1,\ldots,\zeta_{k-1}$, and the first $i$ rows of $\theta$. On the event $\cbr{\rho_{k-1}=i}$, the point $\tilde z_k$ is then fixed, while the $j$-th row of $\theta$ is still uniform on $\cbr{-1,1}^n$. Hence $v_j(\tilde z_k)=\frac4n\sum_{\ell}\theta_{j\ell}\tanh(\sqrt n\,\tilde z_{k,j\ell})$ is a sum of $n$ independent mean-zero random variables, each with absolute value at most $\frac4n$. By Hoeffding's inequality,
\[
\PP\left(|v_j(\tilde z_k)|>\tfrac12,\ \rho_{k-1}=i\right)\le 2e^{-n/128}\,\PP(\rho_{k-1}=i).
\]
Let $\cB$ be the event that $|v_j(\tilde z_k)|>\frac12$ for some $k\le N$ and some $j>\rho_{k-1}$. Summing over $k$, $i$, and $j$ gives $\PP(\cB)\le 2NJe^{-n/128}$.

We now show that the actual and the surrogate interactions coincide on $\cB^c$. Suppose that $z_{k'}=\tilde z_{k'}$ and $\gamma_{k'}=\tilde\gamma_{k'}$ for every $k'<k$. Then $z_k=\tilde z_k$. On $\cB^c$, every $j>\rho_{k-1}$ satisfies $|v_j(z_k)|\le\frac12$, so $\pi(z_k)=\pi_{\rho_{k-1}}(z_k)=\tilde\pi_k$, and hence $\gamma_k=\tilde\gamma_k$. By induction, $z_k=\tilde z_k$ for every $k\le N$.

Finally, we bound $\rho_N$. Since $\rho_k$ increases only if $\zeta_k=1$, we have $\rho_N\le\sum_{k}\zeta_k$. By Markov's inequality, $\PP(\sum_k\zeta_k\ge J)\le\frac{Np}{J}$. On the event $\cB^c\cap\cbr{\sum_k\zeta_k<J}$, every $k\le N$ satisfies $\rho_{k-1}<J$, and hence $|v_J(z_k)|=|v_J(\tilde z_k)|\le\frac12$. This event has probability at least $1-2NJe^{-n/128}-\frac{Np}{J}$. We finish the proof.
\end{proof}

\end{document}